\documentclass[sigconf]{aamas} 
\renewcommand\footnotetextcopyrightpermission[1]{} 
\usepackage{balance} 
\usepackage{lineno}
\usepackage{algorithm}
\usepackage{tikz}
\usetikzlibrary{arrows}
\usepackage{algpseudocode}
\usepackage{subcaption}
\usepackage{booktabs}
\usepackage{multirow}
\usepackage{makecell}
\usepackage{enumitem}
\usepackage{amsmath}
\usepackage{etoolbox}
\usepackage{amsfonts,amsthm,bm,mathtools}
\usepackage{blkarray}
\usepackage[capitalise]{cleveref}

\allowdisplaybreaks

\setcopyright{ifaamas}
\acmConference[AAMAS '24]{Proc.\@ of the 23rd International Conference
on Autonomous Agents and Multiagent Systems (AAMAS 2024)}{May 6 -- 10, 2024}
{Auckland, New Zealand}{N.~Alechina, V.~Dignum, M.~Dastani, J.S.~Sichman (eds.)}
\copyrightyear{2024}
\acmYear{2024}
\acmDOI{}
\acmPrice{}
\acmISBN{}

\acmSubmissionID{325}

\title[DFL for RMABs]{Efficient Public Health Intervention Planning Using Decomposition-Based Decision-Focused Learning}

\author{Sanket Shah}
\authornote{Work done as an intern at Google Research India}
\affiliation{
  \institution{Harvard University}
  \country{}
}
\email{sanketshah@g.harvard.edu}

\author{Arun Suggala}
\affiliation{
  \institution{Google Research India}
  \country{}}
\email{arunss@google.com}

\author{Milind Tambe}
\affiliation{
  \institution{Google Research}
  \country{}}
\email{milindtambe@google.com}

\author{Aparna Taneja}
\affiliation{
  \institution{Google Research India}
  \country{}}
\email{aparnataneja@google.com}

\begin{abstract}
The declining participation of beneficiaries over time is a key concern in public health programs. A popular strategy for improving retention is to have health workers `intervene' on beneficiaries at risk of dropping out. However, the availability and time of these health workers are limited resources. As a result, there has been a line of research on optimizing these limited intervention resources using Restless Multi-Armed Bandits (RMABs). The key technical barrier to using this framework in practice lies in the need to estimate the beneficiaries' RMAB parameters from historical data. Recent research has shown that Decision-Focused Learning (DFL), which focuses on maximizing the beneficiaries' adherence rather than predictive accuracy, improves the performance of intervention targeting using RMABs.
Unfortunately, these gains come at a high computational cost because of the need to solve and evaluate the RMAB in each DFL training step.
In this paper, we provide a principled way to exploit the structure of RMABs to speed up intervention planning by cleverly decoupling the planning for different beneficiaries. We use real-world data from an Indian NGO, ARMMAN, to show that our approach is up to two orders of magnitude faster than the state-of-the-art approach while also yielding superior model performance. This would enable the NGO to scale up deployments using DFL to potentially millions of mothers, ultimately advancing progress toward UNSDG 3.1.
\end{abstract}


\DeclareMathOperator{\E}{\mathbb{E}}

\newtheorem{thm}{Theorem}
\newtheorem*{thm*}{Theorem}

\newenvironment{pfsketch}{%
    \renewcommand{\proofname}{Proof Sketch}\proof}{\endproof}
\theoremstyle{definition}
\newtheorem{exmp}{Example}[section]
\AtBeginEnvironment{exmp}{%
  \pushQED{\qed}%
}
\AtEndEnvironment{exmp}{\popQED\endexmp}

\DeclareMathOperator*{\argmax}{arg\,max}

\definecolor{myOrange}{HTML}{E67E27}
\definecolor{myBlue}{HTML}{6091DB}
\definecolor{myRed}{HTML}{BC1708}
\newcommand{\tgood}{{\color{ForestGreen} T^{\text{good}}}}
\newcommand{\tbad}{{\color{Maroon} T^{\text{bad}}}}
\newcommand{\rbar}{\textcolor{myRed}{\Bar{R}}}
\newcommand{\bmthat}{\textcolor{myOrange}{\bm{\hat{T}}}}
\newcommand{\bmtruet}{\textcolor{myBlue}{\bm{T}}}
\def\that{\bgroup\colorlet{outcolor}{.}\color{myOrange}\hat{T}\futurelet\next\parameterA}
\def\truet{\bgroup\colorlet{outcolor}{.}\color{myBlue}T\futurelet\next\parameterA}
\def\jbar{\bgroup\colorlet{outcolor}{.}\color{myRed}\Bar{J}\futurelet\next\parameterA}
\def\parameterA{\ifx\next_\expandafter\parameterB\else\egroup\fi}
\def\parameterB_#1{_{{\color{outcolor}#1}}\futurelet\next\parameterC}
\def\parameterC{\ifx\next^\expandafter\parameterD\else\egroup\fi}
\def\parameterD^#1{^{\color{outcolor}#1}\egroup}

\makeatletter
\gdef\@copyrightpermission{
	\begin{minipage}{0.3\columnwidth}
		\href{https://creativecommons.org/licenses/by/4.0/}{\includegraphics[width=0.90\textwidth]{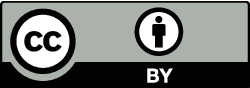}}
	\end{minipage}\hfill
	\begin{minipage}{0.7\columnwidth}
		\href{https://creativecommons.org/licenses/by/4.0/}{This work is licensed under a Creative Commons Attribution International 4.0 License.}
	\end{minipage}
	\vspace{5pt}
}
\makeatother

\begin{document}


\pagestyle{fancy}
\fancyhead{}


\maketitle 


\section{Introduction}
A pervasive challenge faced by public health programs is one of beneficiary retention. To combat the declining engagement of beneficiaries over time, a common strategy has been to use `interventions' (e.g., personalized service calls) to encourage participation and address concerns. This has been employed in a variety of domains such as medication adherence~\cite{mate2020collapsing}, chronic illness management~\cite{killian2023equitable}, treatment prioritization~\cite{ayer2019prioritizing}, and mobile health~\cite{mate2022field}. However, despite their effectiveness, such interventions are expensive and, thus, effectively limited resources.
Consequently, optimizing the selection of beneficiaries for these interventions is crucial.

Towards this end, there has been a recent line of research on using Restless Multi-Armed Bandits (RMABs)~\cite{whittle1988restless,weber1990index,jung2019regret} to optimize intervention resources in these domains. In the RMAB framework, each beneficiary's adherence to the program is modeled as a Markov Decision Process (MDP). The goal, then, is to design policies that choose $K$ out of $N$ beneficiaries for health worker intervention in each timestep such that the overall adherence of all beneficiaries is maximized. However, the key technical barrier to using this framework in practice lies in estimating the beneficiaries' MDP parameters, which are essential for determining these intervention policies. To address this gap, past work relies on \textit{predicting} these parameters using historical data and beneficiary demographics.

An essential component of an effective predictive pipeline in the public health domain involves using `Decision-Focused Learning' (DFL)~\cite{elmachtoub2022smart,wilder2019melding,mandi2022decision}, a way to incorporate intervention planning into the training loop in order to create models that maximize beneficiary adherence directly (cf. predictive accuracy). Both simulated experiments~\cite{wang2023scalable,killian2019learning} and a field study~\cite{verma2023restless} have shown that models trained using DFL outperform those trained using traditional supervised learning pipelines. However, the improved performance of DFL comes at a heavy computational cost---incorporating decision-making into the training pipeline requires solving, evaluating, and differentiating through intervention planning at every training step.

To reduce the computational overhead of using DFL, the state-of-the-art approach~\cite{wang2023scalable} uses the popular Whittle Index heuristic~\cite{weber1990index} to simplify intervention planning. This heuristic decomposes the task of creating a good policy for \textit{all the beneficiaries} to one of deciding whether to act on \textit{individual beneficiaries} in a simplified version of the RMAB problem. However, while this speeds up the \textit{planning} of a good policy, \textit{evaluating} the resulting policy requires repeatedly simulating the outcome of the policy. Yet, such evaluation is a crucial aspect of the DFL training pipeline. Indeed, as we show in \cref{sec:results}, this either results in evaluations with high variance and, as a result, suboptimal learning (for a small number of simulations), or high cost (for a large number of simulations).

Instead, in this paper, we create a decomposition-based DFL approach that extends the ideas from the RMAB planning literature~\cite{weber1990index,hawkins2003langrangian} to both create \textit{and evaluate} policies efficiently, \textit{without the need for any simulations}. Specifically, we begin in~\cref{sec:constrviolation} by showing how using the approach from~\citet{hawkins2003langrangian} to create decomposed policies leads to budget constraint violations in the DFL setting.
Rather, in~\cref{sec:feasibility}, we propose an alternative approach and show how optimizing over a richer class of policies allows us to provably estimate the optimal beneficiary parameters in this setting. Finally, in~\cref{sec:fasteval}, we show how to efficiently (in $O(N)$ time) incorporate this approach into the DFL pipeline by building on techniques from the DFL literature~\cite{amos2017optnet,amos2019limited}. 

To evaluate our approach, we use real-world data from ARMMAN, an Indian NGO, that leverages mobile health (mHealth) technology to promote healthy pregnancies. Specifically, we use secondary data from their mMitra program~\cite{mmitra}, which has successfully delivered vital preventive care information to \textit{2.9 million} women, to build our domain. Notably, DFL~\cite{verma2023restless} has been currently deployed for intervention planning in mMitra and has served around 250,000 beneficiaries so far. Then, in~\cref{sec:results}, we present the results of how our approach does against this existing approach (based on~\citet{wang2023scalable}) on both the real-world domain and a synthetic domain. 

\textbf{We show that our proposed method is up to 500x faster than the currently deployed approach}, while also producing better-performing models (\cref{tab:results}). Practically, this means that models that would take more than a day to train in the past can now be trained in minutes \textit{with no loss in quality}. All in all, we believe that our contribution will allow more scalable learning for RMABs, and hopefully help ARMMAN and other NGOs move us one step closer to UN Sustainable Development Goal 3.1.

\section{Background}\label{sec:background}
\begin{figure}
    \centering
    \includegraphics[width=0.8\linewidth]{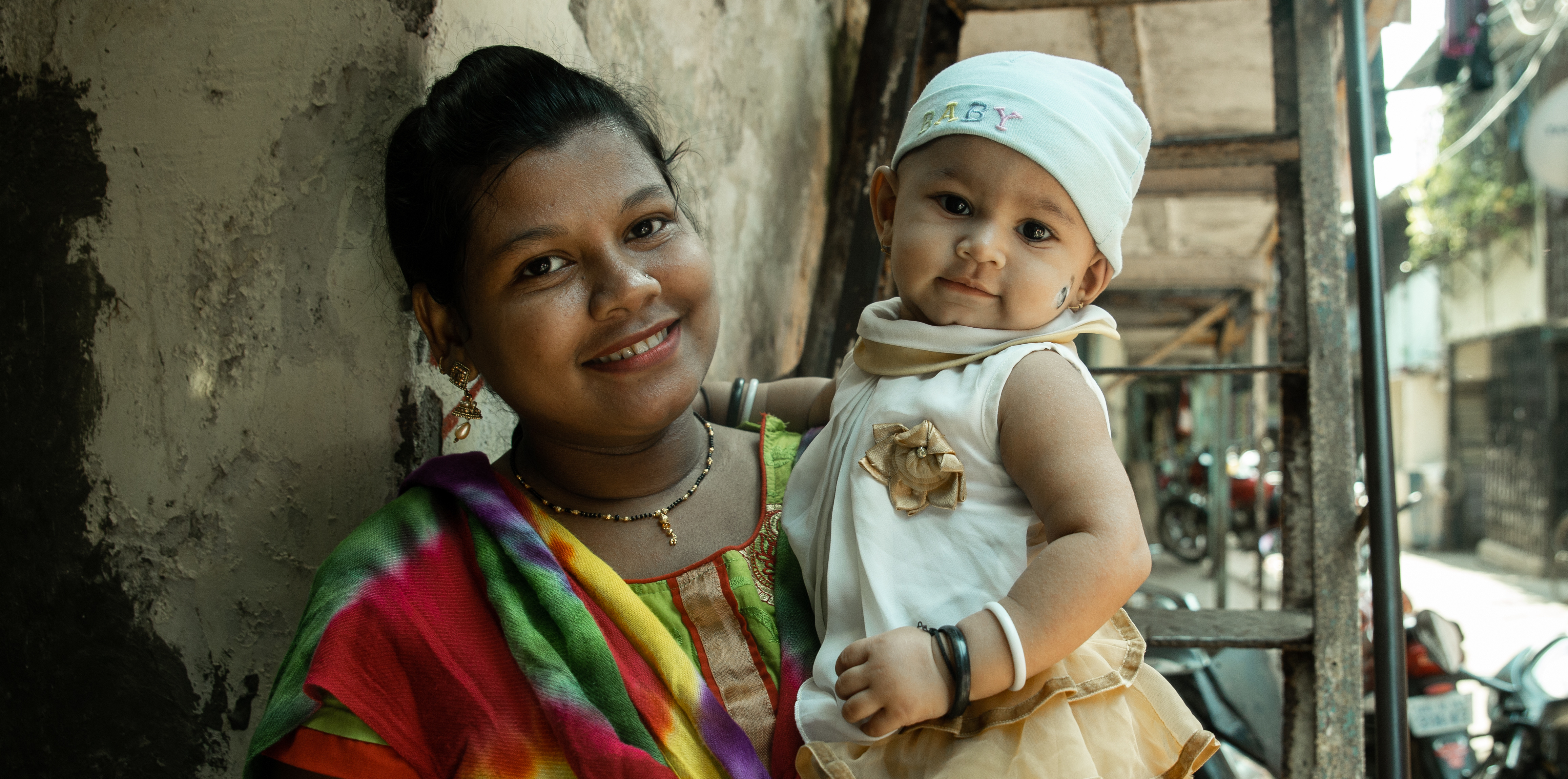}
    \vspace{-0.5em}
    \caption{An mMitra beneficiary \textmd{(courtesy of ARMMAN)}}
    \vspace{-1em}
    \label{fig:benef}
\end{figure}

\subsubsection*{ARMMAN's mMitra Program}
The UN Sustainable Development Goal (SDG) 3.1 aims to reduce the global maternal mortality ratio to below 70 per 100,000 live births by 2030. In line with this goal, ARMMAN uses mHealth technology to combat maternal and neonatal mortality in underprivileged communities across India. Specifically, ARMMAN's mMitra program delivers preventive care information on maternal and infant health through free automated voice calls to beneficiaries. Notably, $\approx$90\% of mothers in the program fall below the World Bank's international poverty line~\cite{verma2023increasing}. Consequently, these weekly calls provide vital and timely information that would otherwise remain inaccessible to these women. However, despite the program's success, engagement wanes over time, with 22\% of beneficiaries dropping out within just three months of enrollment~\cite{verma2023increasing}. To combat this, ARMMAN deploys health workers to conduct live service calls to encourage participation and address concerns. In the context of the mMitra program, our goal is to determine which subset of beneficiaries to select for these service calls on a weekly basis so as to maximize engagement. 

\subsubsection*{Restless Multi-Armed Bandits}
RMABs are an extension of the well-known multi-armed bandit framework to the case where the states of different arms evolve over time \textit{regardless of whether they are pulled or not}. Concretely, each arm $i \in [N]$ of the RMAB is modeled as an MDP that is defined by the tuple $(\mathcal{S}_i, \mathcal{A}_i, T_i, R_i, \gamma)$ where $\mathcal{S}_i$ is the state space, $\mathcal{A}_i$ is the action space, $T_i, R_i \colon \mathcal{S}_i \times \mathcal{A}_i \times \mathcal{S}_i \to \mathbb{R}$ are the transition and reward functions, and $\gamma$ is the discount factor.

Although the results presented in this paper extend to all RMABs, we make the following simplifications for ease of exposition:
\begin{itemize}[leftmargin=1em,nosep]
    \item $\mathcal{S}_i \coloneqq \mathcal{S} = \{0,\ldots,|\mathcal{S}| - 1\}$ that denotes the degree of engagement with the public health program.
    \item $R_i \coloneqq R(s) = \frac{s}{|S| - 1}$, the reward is directly proportional to the degree of engagement with the program.
    \item $\mathcal{A}_i = \mathcal{A} \coloneqq \{0,1\}$ that denotes whether a beneficiary is intervened on (1) or not (0).
\end{itemize}
An important point to note here is that, in large-scale public health interventions, we typically do not have enough data to estimate complex per-arm models, especially for the intervention action. As a result, each per-arm MDP (i.e., $|\mathcal{S}|$) is typically small.

The solution concept for RMABs is a policy $\pi \colon \mathcal{S}^N \to \mathcal{A}^N$ that satisfies a budget constraint $\sum_i \pi_i(s_i) \leq B$ where $B$ is our budget. The optimal policy $\pi^\star$ for transitions $\bm{T}$ can then be written as:
\begin{align}
    \bm{\pi}^\star(\bm{T}) = \argmax_{\bm{\pi}} J_{\bm{T}}(\bm{\pi}) \quad
    s.t. \;\; \sum_{i = 1}^N \pi_i(s_i) \leq B,\; \forall \bm{s} \in \mathcal{S}^N \label{eqn:exact}
\end{align}
where $J_{\bm{T}}(\bm{\pi}) = \mathbb{E}_{\tau \sim \bm{\pi}, \bm{T}}[R(s) + \gamma R(s') + \gamma^2 R (s'') + \ldots]$ is the expected return for trajectories $\tau$ generated using policy $\bm{\pi}$ and transitions $\bm{T}$.

In the RMAB above, \emph{the only thing that is unknown} is the transition matrix $\bm{T}$ that determines beneficiaries' engagement and response to interventions. The challenge, then, is estimating $\bm{T}$. 

\subsubsection*{Decision-Focused Learning (DFL)} While the parameters in bandit problems are sometimes learned online, in public health settings this can be impractical because the programs are short and feedback infrequent. For example, ARMMAN's mMitra program runs for 72 weeks and beneficiaries are only called once a week. 
Moreover, we want to be able to intervene \textit{as early as possible} to prevent beneficiaries from dropping out of the program. As a result, we instead estimate the transition matrices $\bmtruet$ from historical data, \textit{offline}.

This has been modeled as a Predict-then-Optimize (PtO) problem in past work~\cite{wang2023scalable,verma2023restless} and involves three steps:
\begin{enumerate}[leftmargin=1.5em,nosep]
    \item \textbf{Predict Step:} First, we use the demographic features $\bm{x} = [x_1, \ldots, x_N]$ associated with each of the $N$ beneficiaries (arms) to \textit{predict} their transition matrices $\bmthat = [\that_1, \ldots, \that_N] = [M_\theta(x_1), \ldots, M_\theta(x_N)]$ using a predictive model $M_\theta$. 
    \item \textbf{Optimize/Planning Step:} Next, we use these predicted transition matrices $\bmthat$ to compute the optimal policy $\bm{\pi}^\star(\bmthat) = \max_{\bm{\pi}} J_{\bmthat}(\bm{\pi})$, where $J$ is the expected return under policy $\bm{\pi}$.
    \item \textbf{Evaluation Step} Finally, we evaluate the policy $\bm{\pi}^\star(\bmthat)$ on the \textit{true} historical transition probabilities $\bmtruet$, i.e., $J_{\bmtruet}(\bm{\pi}^\star(\bmthat))$. We call this value the `Decision Quality' (DQ) of the prediction $\bmthat$.
\end{enumerate}
\vspace{0.3em}
The overall goal for DFL, then, is to learn a set of parameters $\theta^\star$ for the predictive model $M_\theta$ such that the final decision quality is maximized. With a slight abuse of notation where $M_\theta(\bm{x}) = [M_\theta(x_1), \ldots, M_\theta(x_N)]$, this can be written as:
\begin{align}
    \theta^\star = \argmax_{\theta}\; \E_{\bm{x}, \bmtruet \sim \mathcal{D}} \underbrace{\left [ J_{\bmtruet}(\bm{\pi}^\star(M_\theta(\bm{x}))) \right ]}_{\ell_{\text{DFL}}(M_\theta(\bm{x}), \bmtruet)} \label{eqn:dfl}
\end{align}
This is different from a typical supervised learning problem in which the goal is to minimize a ``standard'' loss function, e.g., MSE:
\begin{align*}
    \theta^\star = \argmax_{\theta}\; \E_{\bm{x}, \bmtruet \sim \mathcal{D}} \underbrace{\left [ ||M_\theta(\bm{x}) - \bmtruet||_2^2 \right ]}_{\ell_\text{MSE}(M_\theta(\bm{x}), \bmtruet)}
\end{align*}

\section{Related Work}\label{sec:related}
\subsubsection*{DFL for RMABs} The closest related branch of the literature on solving problems similar to~\cref{eqn:dfl} is that of decision-focused model-based
reinforcement learning ~\cite{futoma2020popcorn,wang2021learning,farahmand2017value,nikishin2022control}. There, the goal is to estimate MDP parameters that lead to good downstream policies. However, while these approaches can \textit{technically} be applied to the RMAB domain, the state space of RMABs is combinatorial in the number of arms $N$ and known to be PSPACE-Hard~\cite{papadimitriou1994complexity} to solve.

To make solving \cref{eqn:dfl} computationally tractable for RMABs, \citet{wang2023scalable} propose an efficient approximate approach to planning that uses the popular Whittle Index-based policy~\cite{weber1990index}:
\begin{align}
    \bm{\pi}^\star \approx \bm{\pi}^\text{WI} (s) = \begin{cases}
        1 & \text{WI}^i(s^i) \in \text{Top-}$B$(\text{WI}(\bm{s})) \\
        0 & \text{otherwise}
    \end{cases} \label{eqn:wi} 
\end{align}
However, while $\bm{\pi}^{\text{WI}}$ only depends on the Whittle Indexes (that are calculated independently per arm), the Top-$B$ policy still acts on the combinatorial state space $\bm{s} = [s_1, \ldots, s_N]$. As a result, evaluating $\bm{\pi}^{\text{WI}}$ requires using expensive Monte Carlo simulations. Instead, in this paper, we propose a novel and significantly cheaper way to approximate both policy creation \textit{and evaluation}.

\subsubsection*{Decomposed RMAB Evaluation} The solution we present in \cref{sec:decompeval} builds on foundational work in the planning literature~\cite{weber1990index,hawkins2003langrangian}. The Whittle Index heuristic itself is based on a relaxation of~\cref{eqn:exact} that decomposes the combinatorial problem into $N$ per-arm problems. However, in \cref{sec:constrviolation}, we describe why existing methods lead to constraint violations in our DFL setting. Then, in \cref{sec:feasibility}, we show how to modify these ideas so that they \textit{are} applicable and derive a novel solution method for the resulting formulation.

\subsubsection*{Multi-Model MDPs} Our solution in \cref{sec:feasibility} requires coming up with a policy that maximizes the return with respect to one MDP (A) while having a bounded return with respect to a different MDP (B). This is a generalization of the popular ``Constrained MDPs'' framework~\cite{altman1999constrained} to the case where the MDPs A and B have different transition matrices in addition to different reward functions. The most directly related work to this is that of ``Concurrent MDPs'' ~\cite{buchholz2019computation} or ``Multi-model MDPs''~\cite{steimle2021multi}, which show that solving for such policies is NP-Hard and provide Mixed Integer Programming-based solutions. Instead, in this paper, we use the fact that per-arm MDPs for public health RMABs are typically small to create an efficient alternate approach that is also easily differentiable.

\section{Decomposed RMAB Evaluation} \label{sec:decompeval}
Our high-level idea for speeding up DFL involves coming up with a good policy $\bm{\pi}^{\text{DEC}}$ that has the following properties:
\begin{itemize}[leftmargin=1em,nosep]
    \item \textbf{Decomposable:} If we can come up with a good policy $\bm{\pi}^{\text{DEC}} =[\pi^{\text{DEC}}_1(s_1), \ldots, \pi^{\text{DEC}}_N(s_N)]$ that acts on different beneficiaries independently, we can also \textit{evaluate} it in a decomposed manner:
    \begin{align*}
    J_{\bmtruet}(\bm{\pi}^{\text{DEC}}) = \sum_i J_{\truet_i}(\pi^{\text{DEC}}_i)
    \end{align*}
    Specifically, we can evaluate the per-arm returns by solving the Bellman Equations (\cref{alg:returns} in \cref{sec:returns}) \textit{without the need for simulations}, because the number of states in each per-arm MDP is typically small in RMAB formulations for public health.
    \item \textbf{Differentiable:} If the algorithm for estimating $\bm{\pi}^{\text{DEC}}$ is differentiable, we can simply substitute $\bm{\pi}^\star$ with $\bm{\pi}^{\text{DEC}}$ in \cref{eqn:dfl} to get the following decomposed estimator for the predictive model:
\begin{align}
    \theta^\star = \argmax_{\theta}\; \E_{\bm{x}, \bmtruet \sim \mathcal{D}} \left [ J_{\bmtruet}(\bm{\pi}^{\text{DEC}}(M_\theta(\bm{x}))) \right ] \label{eqn:naivedecomp}
\end{align}
\end{itemize}
Importantly, the Whittle Index policy $\bm{\pi}^{\text{WI}}$ in \cref{eqn:wi} that is used by~\citet{wang2023scalable} is \textit{not} decomposable because we need to know the states $\bm{s}$ of all beneficiaries to determine $\text{Top-}$B$(\text{WI}(\bm{s}))$ (\cref{eqn:wi}). 

In the remainder of this section, we begin by showing why past approaches for calculating $\bm{\pi}^{\text{DEC}}$ lead to bad estimators of $\theta^\star$  and hence bad estimates of $\bmthat$ in 
\cref{sec:constrviolation}. Then, in \cref{sec:feasibility}, we propose an alternate problem formulation that leads to provably good estimation. Finally, we show how to efficiently solve for $\bm{\pi}^{\text{DEC}}$ in this alternative formulation by extending techniques from the DFL literature in \cref{sec:fasteval}.

\subsection{Limitations of Past Work in Estimating $\theta^\star$}\label{sec:constrviolation}
To create a policy that does not depend on the joint state $\bm{s} = [s_1, \ldots, s_N]$ of all the beneficiaries but rather on each beneficiary individually, past work~\cite{weber1990index,hawkins2003langrangian} relaxes the per-state budget constraint in \cref{eqn:exact} to a constraint over the amount of budget used \textit{in expectation}. This results in the following relaxed problem:
\begin{align}
    \bm{\pi}^{\that\text{-DEC}} (\bmthat) = \argmax_{\bm{\pi}} \; J_{\bmthat}(\bm{\pi}) \quad s.t. \;\; \jbar_{\bmthat}(\bm{\pi}) \leq \frac{B}{1-\gamma} \label{eqn:relaxed}
\end{align}
where, $\jbar$ is the expected return of an MDP with transitions $\bmthat$, but a different reward $\rbar(\bm{s}, \bm{a}) = \sum_{i \in [N]} a_i$. $\jbar$ keeps track of how many interventions the policy $\bm{\pi}$ performs, and the constraint makes sure that this value is bounded by the (infinite-horizon discounted) budget $\frac{B}{1 - \gamma}$. Then, \citet{hawkins2003langrangian} shows an efficient way to solve the dual reformulation of this problem to get a decomposable policy.

However, while all the planning literature only focuses on calculating a good policy for a \textit{single} fixed transition matrix $\bm{T}$, there are actually two sets of transition matrices in our DFL setting---the predicted transition matrices $\bmthat$, and the true transition matrices $\bmtruet$. As a result, if we use \cref{eqn:relaxed} to plan for the optimal policy $\bm{\pi}^\star(\bmthat) \approx \bm{\pi}^{\that\text{-DEC}}(\bmthat)$ in the DFL pipeline, we would only satisfy the budget constraint with respect to the \textit{predicted} transitions, not the true transitions. As a result, if $\bmthat \neq \bmtruet$ it could lead to (possibly large) constraint violations:

\begin{exmp}\label{ex:violation}
Below, we describe what may go wrong in the simplest possible parameter estimation problem---predicting the parameters of an RMAB with only one arm, i.e., a single 2-state MDP's transition matrix $\that$. Consider the prediction:
\begin{gather}
    \that^{0} = 
    \begin{bmatrix}
        1 & 0 \\
        0 & 1
    \end{bmatrix},\;
    \that^{1} = 
    \begin{bmatrix}
        0 & 1 \\
        0 & 1
    \end{bmatrix},\;
    \text{WI}_{\that} =  \begin{bmatrix}
        \frac{\gamma}{1 - \gamma} \\
        0
    \end{bmatrix}
    \label{eqn:bestpred}
\end{gather}
where, an entry of the matrix $\that^{a}_{s,s'}$ represents the probability $P(s'|s, a)$ of transitioning to state $s'$ when in state $s$ and taking action $a$, and $\text{WI}_{\that}$ contains the whittle indices of each state.

This MDP has the highest possible Whittle Index (action effect) for state 0---if you don't act, you'll always stay in state 0 and accumulate no reward, but if you act on the arm \textit{just once}, you will transition to state 1 where you can passively collect a reward of 1 in every timestep without ever needing to act again. Because you only need to act once to get the benefits, the optimal policy uses only 1 unit of budget in comparison to the $\frac{B}{1 - \gamma}$ units that are available (the $1 - \gamma$ factor comes from the infinite-horizon discounting). As a result, as long as our budget $B \geq 1 - \gamma$, the optimal policy $\pi^{\that\text{-DEC}}(\that)$ according to~\cref{eqn:relaxed} will be to act in state 0.

However, in the DFL context, this policy must be evaluated not on the predicted transition matrix $\that$, but on the \textit{true} transition matrix that could be completely different. For example, consider the following true transition matrix $\bmtruet$:
\begin{gather*}
    \truet^{0} = 
    \begin{bmatrix}
        1 & 0 \\
        1 & 0
    \end{bmatrix},\;
    \truet^{1} = 
    \begin{bmatrix}
        1 & 0 \\
        1 & 0
    \end{bmatrix},\;
    \text{WI}_{\truet} =
    \begin{bmatrix}
        0\\
        0
    \end{bmatrix}
\end{gather*}
For this transition matrix, we will always stay in state 0 (or move there, if we start in state 1). Applying the policy $\pi^{\that\text{-DEC}} (\that)$ from above, that chooses to act in state 0, we will expend a discounted budget of $\approx \frac{1}{1 - \gamma}$ because we will act in every timestep. As a result, if our true budget is only $1 - \gamma$, we will overshoot our budget by a factor of $\frac{\text{used budget}}{\text{true budget}} = \frac{\frac{1}{1 - \gamma}}{1 - \gamma} = \frac{1}{(1 - \gamma)^2}$, which is \textbf{100x} for a standard discount factor of $\gamma = 0.9$.
\end{exmp}

The example above shows that there exists a combination of predictions $\bmthat$ and true matrices $\bmtruet$ for which~\cref{eqn:relaxed} leads to budget violations. However, the goal is not to solve for good policies, but rather to estimate parameters by using~\cref{eqn:relaxed} in the DFL pipeline. So, do these budget violations lead to bad parameter estimation? In the theorem below, we show that using~\cref{eqn:relaxed} to perform parameter estimation leads to spurious minima in the DFL setting.

\begin{thm}\label{thm:badminima}
  Predicting $\bmthat = \bmtruet$ is \textbf{not} always a maximizer of the Predict-Then-Optimize problem below:
    \begin{align*}
        \textcolor{myOrange}{\bmthat^\star} = \argmax_{\bmthat}\;  J_{\bmtruet}(\bm{\pi}^{\that\text{-DEC}}(\bmthat))
    \end{align*}
\end{thm}
\begin{pfsketch}
    The intuition for this claim is that, along the lines of \cref{ex:violation}, one can ``buy" more budget by predicting a transition matrix $\bm{\Tilde{T}}$ that uses less budget than the true transitions $\bmtruet$. To prove this, we provide a proof by counterexample where:
    \begin{align*}
        J_{\bmtruet}(\bm{\pi}^{{\that\text{-DEC}}}(\bm{\Tilde{T}})) > J_{\bmtruet}(\bm{\pi}^{{\that\text{-DEC}}}(\bmtruet)) & \qedhere
    \end{align*}
\end{pfsketch}
Moreover, our choice of $\bmtruet$ in the counter-example is not special, making bad parameter estimation the norm, and not an exception.

\subsection{Our Approach: DEC-DFL}\label{sec:feasibility}
In this section, we begin by proposing~\cref{eqn:relaxed-corrected}, an alternative to to~\cref{eqn:relaxed}, that leads to provably good parameter estimation (\cref{thm:goodminima}). Then, to solve~\cref{eqn:relaxed-corrected}, we propose a series of approximations that exploit the properties of~\cref{thm:goodminima} and the fact that per-arm MDPs in public health-based RMABs are small, to get~\cref{alg:decdfl}.

Then, we begin this section by first defining an alternative to~\cref{eqn:relaxed} that ensures budget feasibility:
\begin{align}
    \bm{\pi}^{\truet\text{-DEC}} (\bmthat) = \argmax_{\bm{\pi}} \; J_{\bmthat}(\bm{\pi}) \quad s.t. \;\; \jbar_{\bmtruet}(\bm{\pi}) \leq \frac{B}{1-\gamma} \label{eqn:relaxed-corrected}
\end{align}
where the only difference is that the budget constraint must now be satisfied with respect to \textit{true transition matrix} $\,\bmtruet$. Then, we can show that $\bm{\pi}^{\truet\text{-DEC}}$ leads to good DFL parameter estimation:

\begin{thm}\label{thm:goodminima}
    Predicting $\bmthat = \bmtruet$ is always a maximizer of the Predict-Then-Optimize problem below:
    \begin{align}
        \textcolor{myOrange}{\bmthat^\star} = \argmax_{\bmthat}\;  J_{\bmtruet}(\bm{\pi}^{\truet\text{-DEC}}(\bmthat)) \label{eqn:decompdfl}
    \end{align}
\end{thm}
\begin{proof}
    We begin by noting that the input to $J_{\bmtruet}$ in \cref{eqn:decompdfl} is the output of \cref{eqn:relaxed-corrected}. As a result, any such input policy must satisfy the constraint that $\jbar_{\bmtruet}(\bm{\pi}) \leq \frac{B}{1-\gamma}$. Then, the optimal solution to \cref{eqn:decompdfl} across all possible policies is $\bm{\pi}^\star = \argmax_{\jbar_{\bmtruet}(\bm{\pi}) \leq \frac{B}{1-\gamma}} J_{\bmtruet}(\bm{\pi})$, which is (by definition) exactly the solution to $\bm{\pi}^{\truet\text{-DEC}}(\bmtruet)$! Therefore, any prediction $\bmthat$ can only ever do as well as $\bm{\pi}^{\truet\text{-DEC}}(\bmtruet)$:
    \begin{align*}
        J_{\bmtruet}(\bm{\pi}^{\truet\text{-DEC}}(\bmtruet))
        \geq J_{\bmtruet}(\bm{\pi}^{\truet\text{-DEC}}(\bmthat)), \quad \forall \bmthat & 
        \qedhere
    \end{align*}
\end{proof}

Solving the problem in \cref{eqn:relaxed-corrected}, however, is significantly more challenging than solving \cref{eqn:relaxed} because, unlike in \citet{hawkins2003langrangian}, the dual reformulation of \cref{eqn:relaxed-corrected} cannot be efficiently solved (see `Multi-Model MDPs' in~\cref{sec:related}).
Instead, in this paper, we use a different set of approximations that rely on two observations:
\begin{itemize}[leftmargin=1em,nosep]
    \item \textbf{\cref{thm:goodminima} holds regardless of the domain of $\bm{\pi}$:} Our argument only relies on the fact that $\bm{\pi}^{\truet\text{-DEC}}(\bmtruet)$ maximizes $\argmax_{\bm{\pi}} J_{\bmtruet}(\bm{\pi})$. However, this is true regardless of whether $\bm{\pi}$ is a deterministic policy, a randomized policy, or even some mixture of these. As a result, we will have good parameter estimation regardless of the class of policies that we optimize over.
    \item \textbf{We do not have to solve \cref{eqn:relaxed-corrected} exactly}: Given that our only use of $\bm{\pi}^{\text{DEC}}$ is to estimate good parameters $\bm{\theta}^\star$, we do not have to restrict ourselves to using practically implementable policies. Instead, we can choose a different policy space that is easier to optimize over.
    This is similar to minimizing the MSE as an easy-to-optimize surrogate for the ``0-1'' loss.
\end{itemize}
So, while in practice we may want to optimize over the class of deterministic policies that contains, for e.g., the Whittle Index policy $\bm{\pi}^{\text{WI}}$, we can instead optimize over a richer class---a mixture of deterministic policies $Z$ such that $\bm{\pi} \sim Z$. Then, we use two facts to simplify our optimization. First, we use the following theorem to show that optimizing over this space is equivalent to optimizing over the space of \textit{decomposable} deterministic policies $\bm{\pi} \sim Z^{\text{DEC}}$.

\begin{thm}\label{thm:decomposable}
    Let $\Omega$ be the set of all distributions over deterministic policies, and $\Omega^{\mathrm{DEC}}$ be the set of all distributions over deterministic \emph{decomposable} policies. Consider the following optimization problems:
    \begin{align*}
        \max_{Z \in \Omega} \quad  \E_{\bm{\pi} \sim Z} [J_{\bmtruet}(\bm{\pi})], \quad \text{s.t.}  \;\;  \E_{\bm{\pi} \sim Z }[\jbar_{\bmtruet}(\bm{\pi})] \leq \frac{B}{1-\gamma} \\
         \max_{\mathclap{Z \in \Omega^{\mathrm{DEC}}}} \quad \E_{\bm{\pi} \sim Z} [J_{\bmtruet}(\bm{\pi})],\quad \text{s.t.}  \;\; \E_{\bm{\pi} \sim Z }[\jbar_{\bmtruet}(\bm{\pi})] \leq \frac{B}{1-\gamma}
    \end{align*}
    Then, any maximizer of the latter is also a maximizer of the former.
\end{thm}

Second, we use the fact that each per-arm MDP is typically small in public health-based RMAB formulations (just two states in our real-world domain). Combining these two, we can enumerate all $2^{|\mathcal{S}|}$ deterministic per-arm policies, and then solve for the optimal mixture over them using the following optimization problem:
\begin{align}
    Z^\star(J_{\bmthat}, \jbar_{\bmtruet}) = \argmax_{0 \leq Z_{ij} \leq 1} \quad & \sum_{i=1}^N \sum_{j=1}^{2^{|\mathcal{S}|}} Z_{ij} J_{\that_i}(\pi^j) + \Phi(Z) \nonumber \\
    s.t. \quad & \sum_{j=1}^{2^{|\mathcal{S}|}} Z_{ij} = 1 ,\quad \forall i \nonumber \\ 
    & \sum_{i=1}^N \sum_{j=1}^{2^{|\mathcal{S}|}} Z_{ij} \jbar_{\truet_i}(\pi^j) \leq \frac{B}{1-\gamma} \label{eqn:lpsolver}
\end{align}
where each variable $Z^\star_{ij}$ in the solution is the probability of acting on arm $i$ using policy $\pi^j$. $\Phi(Z)$ is a regularization term that is added to make the solution differentiable with respect to $\bmthat$ (discussed in more detail below). Our overall algorithm for the decomposed evaluation of a set of predictions $\bmthat$ is then described in \cref{alg:decdfl}.

\begin{algorithm}[b]
\caption{Calculation of $\ell_{\text{DEC-DFL}}$ using $\bm{\pi}^{\textcolor{myBlue}{T}\text{-DEC}}(\bmthat)$}
\label{alg:decdfl}
\raggedright
\textbf{Input}: Predicted transition matrices $\bmthat$\\
\textbf{Parameter}: True transition matrices $\bmtruet$\\
\textbf{Output}: $\ell_{\text{DEC-DFL}}(\bmthat, \bmtruet)$
\begin{algorithmic}[1] 
\ForAll{$i \in [N]$ and $\pi^j \in 2^{|\mathcal{S}|}$}\Comment{Ideally, in parallel}
\State Get return of ``reward'' MDP and predicted transitions $\that_i$:
$$J_{\that_i}(\pi^j) \gets \textsc{GetReturns}(\that_i, R, \pi^j)$$
\State Get return of ``reward'' MDP and true transitions $\truet_i$:
$$J_{\truet_i}(\pi^j) \gets \textsc{GetReturns}(\truet_i, R, \pi^j)$$
\State Get return of ``budget'' MDP and true transitions $\truet_i$:
$$\jbar_{\truet_i}(\pi^j) \gets \textsc{GetReturns}(\truet_i, \rbar, \pi^j)$$
\EndFor
\State Solve \cref{eqn:lpsolver} using returns $J_{\bmthat}$ and $\jbar_{\bmtruet}$ calculated above:
$$Z^\star \gets Z^\star(J_{\bmthat}, \jbar_{\bmtruet})$$
\State \textbf{return} $\ell_{\text{DEC-DFL}} = \sum_i \sum_j Z^\star_{ij} \cdot J_{\truet_i}(\pi^j)$
\end{algorithmic}
\end{algorithm}

\subsubsection*{Differentiability}\label{sec:diffeval}
From the perspective of the optimization problem, $J_{\that_i}(\pi^j)$ and $\jbar_{\truet_i}(\pi^j)$ are constants. As a result, if we set $\Phi(Z) = 0$, solving \cref{eqn:lpsolver} reduces to a linear program. However, it has been shown that the solutions of linear programs are not differentiable with respect to their inputs~\cite{elmachtoub2022smart,wilder2019melding} because similar predictions almost always lead to the same decisions. To make the solutions of \cref{eqn:lpsolver} vary smoothly as $\bmthat$ changes, we add a regularization term $\Phi$ (e.g., the $L_2$ norm $||Z||_2$ of the variables or the entropy $H(Z)$) to the objective of the optimization problem. 

\subsection{Efficiently Solving \Cref{eqn:lpsolver}}\label{sec:fasteval}
The previous section provided a way to create good decomposable RMAB policies using an approximation to~\cref{eqn:relaxed-corrected}. However, the crux of the solution,~\cref{alg:decdfl}, involves incorporating the optimization problem in~\cref{eqn:lpsolver} into the DFL pipeline. One way to do this would be to use differentiable optimization packages like Cvxpylayers~\cite{cvxpylayers2019} (DEC-DFL), but this can be slow. Instead, in this section, we use the fact that all the arms are tied together only by the budget constraint to speed up~\cref{alg:decdfl} and create our final \textit{`Fast DEC-DFL'} method for RMAB parameter estimation using DFL.

\subsubsection*{Forward Pass} To solve \cref{eqn:lpsolver}, we first observe that the only thing tying together different arms is a single constraint, i.e., $\sum_{i,j} Z_{ij} \jbar_{\truet_i}(\pi^j) \leq \frac{B}{1-\gamma}$. Moreover, \cref{eqn:lpsolver} is a convex optimization problem that is strictly feasible as long as the budget $B > 0$. Then, because of strong duality via Slater's condition~\cite{boyd2004convex}, we can instead solve the following primal-dual problem:
\begin{align*}
    \min_{\lambda \geq 0}\; \argmax_{0 \leq Z_{ij} \leq 1} \quad & \sum_{i=1}^{N} \sum_{j=1}^{2^{|\mathcal{S}|}} Z_{ij} [ J_{\that_i}(\pi^j) - \lambda \jbar_{\truet_i}(\pi^j)] +\, \alpha H(Z) + \lambda\frac{B}{1 - \gamma}\\
    s.t. \quad & \sum_{j=1}^{2^{|\mathcal{S}|}} Z_{ij} = 1,\quad \forall i
\end{align*}
where, $H(Z) = -\sum_i \sum_j Z_{ij} \log Z_{ij}$ is the entropy of the distribution $Z_i$ over the different possible policies $\pi^j$ and $\alpha$ is the weight of the regularization. Then, the solution to the inner maximization problem is given by the softmax function~\cite{amos2019limited,hsieh2019finding}. Therefore we can simplify our reformulated optimization problem as:
\begin{align}
    \min_{\lambda \geq 0} \sum_{i=1}^{N} \sum_{j=1}^{2^{|\mathcal{S}|}} \Tilde{Z}^\star_{ij}(\lambda) [ J_{\that_i}(\pi^j) - \lambda \jbar_{\truet_i}(\pi^j)] +\, \lambda\frac{B}{1 - \gamma} \nonumber \\
    \text{where, }\Tilde{Z}^\star_{ij}(\lambda) = \text{softmax}_{\pi^j} \left(\frac{ J_{\that_i}(\pi^j) - \lambda \jbar_{\truet_i}(\pi^j) }{\alpha} \right) \label{eqn:mindual}
\end{align}
Now, to solve for the optimal value of the dual variable $\lambda^\star$, we rely on KKT conditions. In particular, it is well known that $\lambda^\star$ satisfies the complementary slackness~\cite{boyd2004convex} condition in~\cref{eqn:kkt_condition_lambda}. Then, to solve~\cref{eqn:mindual}, we use a numerical root-finding algorithm to find the value of $\lambda^\star$ that leads to exactly satisfying the budget constraint. Algorithm~\ref{alg:fwdpass} describes this procedure, and the following theorem proves that it does indeed return the optimal dual variable.
\begin{thm}\label{thm:dual}
    \cref{alg:fwdpass} solves for the optimal dual variable $\lambda^\star$
\end{thm}
\begin{proof}
    Based on KKT conditions, we know that any $\lambda^* \geq 0$ satisfying the following condition is an optimal solution to \cref{eqn:lpsolver}:
    \begin{align}
    \label{eqn:kkt_condition_lambda}
        \lambda^\star \left(\sum_{i=1}^{N} \sum_{j=1}^{2^{|\mathcal{S}|}} Z^\star_{ij}(\lambda)\jbar_{\truet_i}(\pi^j) - \frac{B}{1-\gamma}\right) = 0
    \end{align}
    First, observe that $\sum_{i=1}^{N} \sum_{j=1}^{2^{|\mathcal{S}|}} Z^\star_{ij}(\lambda)\jbar_{\truet_i}(\pi^j)$ decreases monotonically in $\lambda$. This follows from \cref{eqn:mindual} and the properties of softmax (see Proposition~\ref{prop:softmax_monotonicity} for a proof).
    Intuitively, $\lambda$ can be thought of as the ``cost of acting''. Then, as $\lambda \to \infty$ you will never act because the cost is too high, and if $\lambda \to -\infty$ you are incentivized to always act.
    
    Now consider the following equation: $\sum_{i=1}^{N} \sum_{j=1}^{2^{|\mathcal{S}|}} Z^\star_{ij}(\lambda)\jbar_{\truet_i}(\pi^j) - B/(1-\gamma) = 0$. Because of the strict monotonicity of $Z^\star_{ij}(\lambda)$ the equation has a unique root. If this root is positive, then it satisfies the KKT condition in Equation~\eqref{eqn:kkt_condition_lambda} and is hence an optimizer. In this case, the budget constraint is tight. On the other hand, if the root is negative, then the budget constraint has a slack and the unique optimal solution is $\lambda^\star = 0$.
\end{proof}
\cref{alg:fwdpass} exploits the monotonicity of $ \sum_{i,j} Z^\star_{ij}(\lambda)\jbar_{\truet_i}(\pi^j)$ to efficiently find a root. It uses bisection method~\cite{brent2013algorithms} and requires at most $\log\epsilon^{-1}$ calls to $\textsc{EvalLambda}$ to find an $\epsilon$-approximate root.
Consequently, the forward pass takes $O(N\cdot 2^{|\mathcal{S}|} \cdot \log\epsilon^{-1}) $ time because each call to \textsc{EvalLambda} takes $O(N \cdot 2^{|\mathcal{S}|})$ time.

\begin{algorithm}[t]
\caption{\textsc{ForwardPass}}
\label{alg:fwdpass}
\raggedright
\textbf{Inputs}: The Expected Returns $J_{\bmthat}$ and $\jbar_{\bmtruet}$ \\
\textbf{Parameter}: Error tolerance $\epsilon$, Budget $B$, Max reward $R_{max}$\\
\textbf{Output}: Distribution $Z^\star$ over arms $i \in [N]$ and policies $\pi^j$
\begin{algorithmic}[1] 
\Procedure{EvalLambda}{$\lambda$}
\State Compute $\Tilde{Z}^\star(\lambda) \gets \text{softmax}_{\pi^j}([J_{\that_i} - \lambda \jbar_{\truet_i}]), \forall i$
\State \textbf{return} $\sum_{ij} \Tilde{Z}^\star_{ij}(\lambda) \jbar_{\truet_i}(\pi^j) - \frac{B}{1-\gamma}$
\EndProcedure
\vspace{1em}
\State Set interval to $I = [-\frac{R_{\text{max}}}{1 - \gamma}, \frac{R_{\text{max}}}{1 - \gamma}]$ \Comment{$\frac{R_{\text{max}}}{1 - \gamma} 
= \max (J_{\bmthat})$}
\State Run the root-finding algorithm to get the optimal penalty $\lambda^\star$:
$$\lambda^\star \gets \textsc{RootFinder}(\textsc{EvalLambda}, I, \epsilon)$$
\State Ignore constraint if $\lambda^\star < 0$, i.e., the constraint is not violated:
$$\lambda^\star \gets \max(\lambda^\star, 0)$$
\State \textbf{return} $\lambda^\star$, $Z^\star \gets \text{softmax}_{\pi^j}([J_{\that_i} - \lambda^\star \jbar_{\truet_i}]), \forall i$
\end{algorithmic}
\end{algorithm}

\subsubsection*{Backward Pass} The goal of the backward pass is to find the derivatives of the minimizer $Z^\star$ with respect to its inputs, i.e., $\nabla_{J_{\bmthat}} Z^\star$ and $\nabla_{\jbar_{\bmtruet}} Z^\star$. To do this, we differentiate through the KKT conditions of \Cref{eqn:lpsolver} and solve the resulting set of linear equations~\cite{amos2017optnet}. Specifically, for a convex program of the form:
\begin{align*}
    \max_z \quad & q^\top z + H(z) \\
    s.t. \quad & Az = b, Gz \leq h
\end{align*}
we get the following set of linear equations:
\begin{align}\label{eqn:kkt}
    \begin{bmatrix}
        \text{diag}(\frac{-1}{z^\star}) & A^\top & G^\top \text{diag}(\lambda^\star)\\
        A & 0 & 0 \\
        G & 0 & - \text{diag}(Gz^\star - h)
    \end{bmatrix}
    \begin{bmatrix}
        d_z \\
        d_\nu \\
        d_\lambda
    \end{bmatrix} = 
    \begin{bmatrix}
        \frac{\partial \ell}{\partial z^\star} \\
        0 \\
        0
    \end{bmatrix}
\end{align}
where (1) $[d_z, d_\nu, d_\lambda]$ are intermediate variables that relate to the gradients of $\ell$ with respect to the parameters of the optimization problem, and (2) $\frac{\partial \ell}{\partial z^\star}$ is the derivative of the evaluation function with respect to the minimizer $z^\star$ and is the input to the backward pass. Then, given the solution to the set of linear equations above, we can extract the derivatives of interest as follows:
\begin{align*}
    \nabla_q \ell = & \nabla_{J_{\bmthat}} \ell = d_z \\
    \nabla_G \ell = & \nabla_{\jbar_{\bmtruet}} \ell = \lambda^\star (d_z - d_\lambda z^{\star})
\end{align*}
The key challenge in the backward pass is in efficiently solving the set of linear equations in \cref{eqn:kkt}. Given that there are $N \cdot 2^{|\mathcal{S}|} + N + 1$ variables, naively solving these equations would be order $O(N^3)$. However, given the sparsity of the matrix, we can use Gaussian elimination to derive a closed-form solution to \cref{eqn:kkt}.

To do this, we begin by considering the simpler case, where there is no budget constraint. The set of equations in \cref{eqn:kkt} can then be completely decomposed into the following per-arm equations:
\begin{align*}
    \begin{bmatrix}
        \text{diag}(\frac{-1}{Z^\star_i}) & \bm{1}_{2^{|\mathcal{S}|}} \\
        \bm{1}_{2^{|\mathcal{S}|}}^\top & 0 
    \end{bmatrix}
    \begin{bmatrix}
        d_{z_i} \\
        d_{\nu_i} \\
    \end{bmatrix}
    =
    \begin{bmatrix}
        \frac{\partial \ell}{\partial Z^\star_i } \\
        0
    \end{bmatrix}
    = 
    \begin{bmatrix}
        J_{\truet_i} \\
        0
    \end{bmatrix}
\end{align*}
and the reduced row-echelon form of the augmented matrix is:
\begin{align*}
    \left [
    \begin{array}{@{}cc|c@{}}
        \text{diag}(\frac{-1}{Z^\star_i}) & \bm{0}_{2^{|\mathcal{S}|}} & J_{\truet_i} - J_{\truet_i}^\top Z^\star_{i} \\
        \bm{0}_{2^{|\mathcal{S}|}}^\top & 1 & J_{\truet_i}^\top Z^\star_{i}
    \end{array}
    \right ]
\end{align*}
Next, we put the budget constraint back in and rewrite the system of equations as the following augmented matrix:
\begin{align*}
    \left [
    \begin{array}{@{}ccc|c@{}}
        \text{diag}(\frac{-1}{Z^\star}) & \bm{0}_{N \cdot 2^{|\mathcal{S}|} \times N} & \lambda^\star \jbar_{\bmtruet} & 
        {\tiny \begin{bmatrix}
            \vdots \\
            J_{\truet_i} - J_{\truet_i}^\top Z^\star_{i} \\
            \vdots
        \end{bmatrix}} \\
        \bm{0}_{N \times N \cdot 2^{|\mathcal{S}|}} & Id_{N \times N} & \bm{0}_{N} & 
        {\tiny \begin{bmatrix}
            \vdots \\
            J_{\truet_i}^\top Z^\star_{i} \\
            \vdots
        \end{bmatrix}} \\
    \jbar_{\bmtruet}^{\,\top} & \bm{0}_{N}^\top & \xi & 0
    \end{array}
    \right ]
\end{align*}
where $\xi = \frac{B}{1-\gamma} -\sum_{ij} \Tilde{Z}^\star_{ij} \jbar_{\truet_i}(\pi^j)$ is the amount of ``slack'' budget left over. We can then perform Gaussian elimination on the budget constraint and back-substitute to get the values of $[d_z, d_\nu, d_\lambda]$; we do not show the exact calculations here because they're clunky, but this can easily be solved algorithmically. In addition, given that we're performing a constant number of operations on $O(N \cdot 2^{|\mathcal{S}|})$ variables, our backward pass has an $O(N)$ complexity.

\begin{table*}
\centering
\caption{Decision Quality Results. \textmd{We document the performance of linear models trained using various loss functions in the table below. The values in \textbf{bold} represent the highest entries in the column, and those in \textit{italics} are those that are in the 95\% confidence interval of the maximum value. We find that our proposed loss functions consistently outperform the baselines from the literature.}}
\label{tab:results}
\resizebox{\linewidth}{!}{%
\begin{tabular}{cccccccc}
\toprule
&\multirow{2}[3]{*}{\textbf{Loss}} & \multicolumn{3}{c}{\textbf{Normalized Joint Test DQ (↑)}} & \multicolumn{3}{c}{\textbf{Normalized Decomposed Test DQ (↑)}}\\
\cmidrule(lr){3-5} \cmidrule(lr){6-8}
&& Real-World & \makecell{Synthetic\\(2-State)}  & \makecell{Synthetic\\(5-State)} & Real-World & \makecell{Synthetic\\(2-State)}  & \makecell{Synthetic\\(5-State)} \\
\midrule
\multirow{2}{*}{2-Stage} & \makecell{NLL} & 0.04 $\pm$ 0.06 & 0.59 $\pm$ 0.13 & 0.16 $\pm$ 0.04 & -0.27 $\pm$ 0.11 & 0.64 $\pm$ 0.18 & 0.15 $\pm$ 0.07\\
& MSE & 0.10 $\pm$ 0.06 & 0.83 $\pm$ 0.03 & \textbf{0.38 $\pm$ 0.03} & -0.19 $\pm$ 0.10 & 0.86 $\pm$ 0.04 & \textit{0.37 $\pm$ 0.03}\\
\midrule
\multirow{4}{*}{\citet{wang2023scalable}} & \makecell{SIM-DFL (1 trajectory)} & -0.05 $\pm$ 0.07 & 0.78 $\pm$ 0.03 & 0.15 $\pm$ 0.02 & -0.33 $\pm$ 0.08 & 0.82 $\pm$ 0.06 & 0.16 $\pm$ 0.04 \\
& \makecell{SIM-DFL (10 trajectories)} & 0.34 $\pm$ 0.15 & 0.79 $\pm$ 0.03 & 0.16 $\pm$ 0.03 & 0.21 $\pm$ 0.20 & 0.84 $\pm$ 0.06 & 0.15 $\pm$ 0.02\\
& \makecell{SIM-DFL (100 trajectories)} & 0.26 $\pm$ 0.10 & 0.80 $\pm$ 0.02 & 0.19 $\pm$ 0.03 & 0.15 $\pm$ 0.11 & 0.86 $\pm$ 0.03 & 0.19 $\pm$ 0.04\\
& \makecell{SIM-DFL (1000 trajectories)} & \textit{Timeout} & 0.80 $\pm$ 0.02 & 0.18 $\pm$ 0.03 & \textit{Timeout} & 0.84 $\pm$ 0.04 & 0.20 $\pm$ 0.05\\
\midrule
\multirow{3}{*}{Ours} & DEC-DFL (L2) & \textit{0.58 $\pm$ 0.04} & \textbf{0.86 $\pm$ 0.02} & \textit{0.34 $\pm$ 0.04} & \textit{0.50 $\pm$ 0.04} & \textbf{0.91 $\pm$ 0.01} & \textbf{0.38 $\pm$ 0.03}\\
& DEC-DFL (Entropy) & \textbf{0.62 $\pm$ 0.05} & \textbf{0.86 $\pm$ 0.02} & \textit{0.33 $\pm$ 0.04} & \textbf{0.52 $\pm$ 0.05} & \textbf{0.91 $\pm$ 0.01} & \textit{0.35 $\pm$ 0.03}\\
& Fast DEC-DFL (Entropy) & \textit{0.57 $\pm$ 0.12} & \textbf{0.86 $\pm$ 0.02} & \textit{0.33 $\pm$ 0.04} & \textit{0.45 $\pm$ 0.14} & \textbf{0.91 $\pm$ 0.01} & \textit{0.35 $\pm$ 0.03} \\
\bottomrule
\end{tabular}%
}
\end{table*}

\section{Experiments}
\label{sec:results}
In this section, we empirically test our proposed approach on two domains and compare it to baselines from the literature.

\subsubsection*{Real-World Dataset} This is the same dataset used by \citet{wang2023scalable}. We use the data from a large-scale anonymized quality improvement study performed by ARMMAN for 7 weeks~\cite{mate2022field} with beneficiary consent. We choose the cohort that received randomized interventions and randomly split it into 60 training, 20 validation, and 20 test sub-cohorts. Each sub-cohort has $N=76$ beneficiaries and a budget of $B=3$.
For the features $\bm{x}$, we use 44 categorical demographic features captured during program intake, e.g., age, education, and income level. For the transitions, we first create trajectories for each beneficiary from their historical listenership. We do this by discretizing engagement into 2 states---an engaging beneficiary listens to the weekly automated voice message (average length 60 seconds) for more than 30 seconds---and sequencing them to create an array $(s_0, a_0, s_1, \ldots)$. Then, to get the transition matrix for beneficiary $i$, we combine the observed transitions with $P_{\text{pop}}$, a prior created by pooling all the beneficiaries' trajectories together:
\begin{align*}
    T_i(s, a, s') = P_i(s' | s, a) = \frac{\alpha P_{\text{pop}}(s' | s, a) + N(s, a, s')}{\sum_{x \in \mathcal{S}} \alpha P_{\text{pop}}(x | s, a) + N(s, a, x)}
\end{align*}
where $N(s, a, s')$ is the number of times the sub-sequence $s, a, s'$ occurs in the trajectory, and $\alpha = 5$ is the strength of the prior.

\subsubsection*{Synthetic Dataset} We also create a synthetic dataset for which it's easier to control for important hyperparameters, e.g., the number of states $|\mathcal{S}|$ in the per-beneficiary MDP. Here, we generate the transition matrices $T$ uniformly at random. We also generate trajectories of $10$ timesteps based on these transition matrices. Then, to create the features, we pass the transition matrices through a randomly initialized $8-$layer feedforward network with a hidden dimension of $1000$. We then generate $100$ cohorts of $N = 100$ beneficiaries with a budget of $B = 10$ per cohort. We split these cohorts into $20$ train, $20$ validation, and $60$ test sub-cohorts.

\subsubsection*{Baselines} Broadly, we compare against two sets of baselines---(1) ``standard'' regression loss functions that focus on predictive accuracy, and (2) the DFL approach proposed by \citet{wang2023scalable}. For the first, we use the Mean Squared Error between the predicted and true transition matrices (used by \citet{mate2022field}), and the Negative Log Likelihood (NLL) that the predicted transition matrices generate the observed trajectories (used as a baseline by \citet{wang2023scalable}). For the second, we use \citet{wang2023scalable}'s SIM-DFL approach and vary the number of \textit{simulated trajectories} used to evaluate the Whittle Index policy, to show the trade-off between cost and learned model quality. We compare these baselines to our proposed approach (\cref{alg:decdfl}), in which we solve \cref{eqn:lpsolver} using either the Cvxpylayers library~\cite{cvxpylayers2019} (DEC-DFL) or the strategy in \cref{sec:fasteval} (Fast DEC-DFL). We use these different approaches to train a linear predictive model.

\subsubsection*{Evaluation Metrics} We evaluate the quality of our learned models $M_\theta$ using the predict-then-optimize framework (\cref{eqn:dfl}):
\begin{align*}
 \text{DQ}(M_\theta) = \E_{\bm{x}, \bmtruet \sim \mathcal{D}} \left [ J_{\bmtruet}(\bm{\pi}^\star(M_\theta(\bm{x}))) \right ]
\end{align*}
where DQ is the ``decision quality'' of the model. We approximate the value of the expectation using samples from the test set, resulting in the `Test DQ'. In addition, we make the following modifications:
\begin{itemize}[leftmargin=1em,nosep]
    \item \textbf{Policy Approximation:} As discussed in \cref{sec:related}, calculating $\bm{\pi}^\star$ is PSPACE-Hard and so we either evaluate the models using $\bm{\pi}^\star \approx \bm{\pi}^{\text{WI}}$ as in past work~\cite{wang2023scalable} to get the ``Joint DQ'', or $\bm{\pi}^\star \approx \bm{\pi}^{\truet\text{-DEC}}$ to get the ``Decomposed DQ''. We use 1000 trajectories to evaluate the Joint Test DQ in the experiments below.
    \item \textbf{Normalization:} In order to ensure that we're focusing on the \textit{intervention effect}, we linearly re-scale the decision quality such that 0 corresponds to the DQ of never acting and 1 corresponds to acting based on perfect predictions.
\end{itemize}
Putting these together we get our metrics of interest, i.e., the `Normalized Joint Test DQ' and the `Normalized Decomposed Test DQ'. The policy used in practice is $\bm{\pi}^{\text{WI}}$, and so the former metric is a good representation of how well the learned models would do if deployed. The latter is the surrogate we introduce in this paper; measuring this allows us to empirically verify that our proposed objective is well-correlated with the true objective of interest.

\subsubsection*{Hyperparameter Tuning} For our experiments, we vary the learning rate $\text{lr} = \{10^{-2}, 10^{-3}, 10^{-4}, 10^{-5}\}$ and for our approach we also vary the regularization constant $\alpha = \{1, 0.1\}$. All our results are averaged over 10 random train-val-test splits and 5 random model initializations per split. We then choose the hyperparameter value which leads to the lowest loss on the validation set. The final results are presented as ``mean $\pm$ 1 standard error of the mean''.

\begin{figure*}
\centering
\begin{subfigure}[c]{0.57\textwidth}
\centering
\resizebox{\linewidth}{!}{%
\begin{tabular}{cccc}
\toprule
\multirow{2}[3]{*}{\textbf{Loss}} & \multicolumn{3}{c}{\textbf{Time Per Epoch In Seconds (↓)}}\\
\cmidrule(lr){2-4}
& Real-World & \makecell{Synthetic\\(2-State)}  & \makecell{Synthetic\\(5-State)} \\
\midrule
\makecell{NLL} & 0.69 $\pm$ 0.08 & 0.22 $\pm$ 0.05 & 0.24 $\pm$ 0.07\\
MSE & 0.49 $\pm$ 0.03 & 0.15 $\pm$ 0.01 & 0.18 $\pm$ 0.06\\
\midrule
\makecell{SIM-DFL (1 trajectory)} & 18.20 $\pm$ 2.78 & 6.51 $\pm$ 1.01 & 18.18 $\pm$ 0.96\\
\makecell{SIM-DFL (10 trajectories)} & 21.34 $\pm$ 1.90 & 8.83 $\pm$ 1.77 & 19.97 $\pm$ 2.08 \\
\makecell{SIM-DFL (100 trajectories)} & 51.10 $\pm$ 1.57 & 29.33 $\pm$ 11.20 & 34.07 $\pm$ 2.00 \\
\makecell{SIM-DFL (1000 trajectories)} & 503.24 $\pm$ 32.16 & 305.69 $\pm$ 130.77 & 246.48 $\pm$ 57.47 \\
\midrule
DEC-DFL (L2) & 4.63 $\pm$ 0.15 & 2.20 $\pm$ 0.40 & 46.28 $\pm$ 4.00\\
DEC-DFL (Entropy) & 19.51 $\pm$ 2.30 & 11.61 $\pm$ 0.87 & 289.62 $\pm$ 49.61\\
Fast DEC-DFL (Entropy) & 1.07 $\pm$ 0.10 & 0.39 $\pm$ 0.03 & 0.70 $\pm$ 0.16 \\
\bottomrule
\end{tabular}%
}
\caption{Time taken by different methods for a single training epoch.}
\label{tab:time}
\end{subfigure}
\hfil
\begin{subfigure}[c]{0.4\textwidth}
    \includegraphics[width=\linewidth]{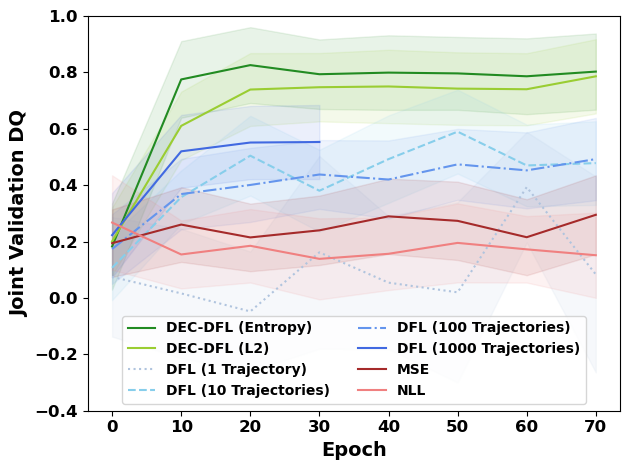}
    \vspace{-1.5em}
    \caption{Validation DQ vs. Epoch on Real-World Dataset.}
    \label{fig:dq_vs_epoch}
\end{subfigure}
\vspace{-0.5em}
\caption{Computational Cost Results. \textmd{In (a), we find that our proposed ``Fast DEC-DFL'' loss is roughly 500x faster than the ``SIM-DFL (1000 Trajectories)'' loss proposed by \citet{wang2023scalable}. In (b), we show that this speed-up does not come at any cost in terms of the rate of convergence. In fact, ``DEC-DFL'' outperforms even ``SIM-DFL (1000 Trajectories)'' till the latter times out.}}
\label{fig:compute}
\end{figure*}

\subsection{Overall Results}\label{sec:overallresults}
In this section, we analyze the results of our experiments, presented in \cref{tab:results} and \cref{fig:compute}. Overall, we find that our `Fast DEC-DFL' approach described in~\cref{sec:fasteval} yields a speed-up of up to 500x over past work while also achieving comparable model performance.

\vspace{0.3em}\noindent We now look more closely at our decision quality results in \cref{tab:results}:
\begin{itemize}[leftmargin=1em,nosep]
    \item \textit{DFL is important in the real-world domain:} The 2-stage methods do significantly worse than both SIM-DFL and DEC-DFL in the real-world domain. This is consistent with past work~\cite{verma2023restless,wang2023scalable}.
    \item \textit{DFL is less useful in the simulated domain:} In the 2-state domain, we find that DEC-DFL performs only slightly better than MSE, and in the 5-stage domain, this difference disappears almost completely. We believe that this is because the true data-generating process is a lot noisier than the one we use to create our synthetic domain, and that is where DFL is particularly useful.
    \item \textit{Decomposed Test DQ mirrors Joint Test DQ:} Broadly, we find that the ordering of methods according to the Decomposed DQ mirrors the ordering according to the Joint DQ, which is used in practice. This suggests that our decomposed evaluation method is a good way to measure decision quality.
    \item \textit{DEC-DFL consistently does better than SIM-DFL:} This is true regardless of the number of trajectories used in SIM-DFL. This, combined with the fact that we continue to see an improvement in DQ as the number of trajectories increases in \cref{fig:dq_vs_epoch}, suggests that even 1000 trajectories are not enough for accurate simulation-based evaluations.
\end{itemize}

\vspace{0.3em}
\noindent We now analyze the computational time results in \cref{fig:compute}. 
\begin{itemize}[leftmargin=1em,nosep]
    \item \textit{Fast DEC-DFL is 500x faster than SIM-DFL (1000 trajectories):} In addition, Fast DEC-DFL even has better performance than SIM-DFL, as seen in \cref{fig:dq_vs_epoch}.
    \item \textit{DEC-DFL does not scale well in $|\mathcal{S}|$:} We see that in going from 2 to 5 states, the computational cost of DEC-DFL increases by $\approx$20x, which is even higher than the $\frac{2^5}{2^2} = 8$ increase in number of policies required to solve \cref{eqn:lpsolver}. This is because naively solving \cref{eqn:kkt} requires inverting a matrix of dimension $O(N\cdot2^{|\mathcal{S}|})$.
    \item \textit{The convergence rate is similar for all methods:} We see in~\cref{fig:dq_vs_epoch} that all the methods seem to converge after a similar number of epochs. This suggests that the per-epoch difference in computational cost from \cref{tab:time} extends to the overall computational cost of training predictive models using the different methods.
\end{itemize}
\vspace{-0.2em}

\subsection{Sensitivity To Model Capacity}
In \cref{sec:overallresults} our results are presented for linear predictive models $M_\theta$. Here, we show that our findings hold \textit{even if we use more complex predictive models}. In \cref{tab:modelcap} we find that:
\begin{itemize}[leftmargin=1em,nosep]
    \item \textit{Increasing model capacity helps 2-stage and SIM-DFL:} We find that increasing model capacity from `small' to `medium' or `large' seems to boost performance when using the `MSE' or `SIM-DFL (1 Trajectory)' losses. However, even a \textit{linear} model trained using DEC-DFL outperforms all other baselines.
    \item \textit{Model capacity does not affect DEC-DFL:} Increasing model capacity does not seem to help when using the DEC-DFL approach. In~\cref{sec:viz}, we visualize the predictions of different approaches and show that DEC-DFL finds beneficiaries that would benefit from interventions even with limited model capacity.
\end{itemize}
\begin{table}
\centering
\caption{Sensitivity to Model Capacity. \textmd{We report the performance of models of varying sizes in the table below.}}
\label{tab:modelcap}
\resizebox{\linewidth}{!}{%
\begin{tabular}{cccc}
\toprule
\multirow{2}[3]{*}{\textbf{Loss}} & \multicolumn{3}{c}{\textbf{Normalized Joint Test DQ (↑)}}\\
\cmidrule(lr){2-4}
& \makecell{Small\\(Linear)} & \makecell{Medium\\(2-Layer, 64 Dim)} & \makecell{Large\\(4-Layer, 500 Dim)}\\
\midrule
\makecell{NLL} & 0.04 $\pm$ 0.06 & 0.01 $\pm$ 0.07 & 0.04 $\pm$ 0.06 \\
MSE & 0.10 $\pm$ 0.06 & 0.35 $\pm$ 0.12 & 0.34 $\pm$ 0.11 \\
\midrule
\makecell{SIM-DFL\\\textcolor{black!40}{(1 trajectory)}} & -0.05 $\pm$ 0.07 & 0.36 $\pm$ 0.07 & 0.30 $\pm$ 0.18 \\
\makecell{SIM-DFL\\\textcolor{black!40}{(10 trajectories)}} & 0.34 $\pm$ 0.15 & 0.47 $\pm$ 0.20 & 0.36 $\pm$ 0.27\\
\makecell{SIM-DFL\\\textcolor{black!40}{(100 trajectories)}} & 0.26 $\pm$ 0.10 & 0.44 $\pm$ 0.21 & 0.33 $\pm$ 0.28 \\
\midrule
\makecell{DEC-DFL\\\textcolor{black!40}{(L2 Reg})} & 0.58 $\pm$ 0.04 & 0.59 $\pm$ 0.04 & 0.61 $\pm$ 0.03\\
\makecell{DEC-DFL\\\textcolor{black!40}{(Entropy)}} & \textbf{0.62 $\pm$ 0.05} & \textbf{0.60 $\pm$ 0.06} & \textbf{0.62 $\pm$ 0.04}\\
\makecell{Fast DEC-DFL\\\textcolor{black!40}{(Entropy)}} & 0.57 $\pm$ 0.12 & \textbf{0.60 $\pm$ 0.04} & 0.61 $\pm$ 0.04\\
\bottomrule
\end{tabular}%
}
\end{table}
\section{Conclusion and Future Work}
Overall, we propose a novel approach, `Fast DEC-DFL', for solving RMABs in the DFL setting. Our approach efficiently calculates decomposable policies that are cheap to evaluate. This results in a \textbf{500x} speedup over state-of-the-art methods on real-world data from ARMMAN, while also improving model performance. Concretely, where past work (`SIM-DFL (1000) Trajectories') can take more than a day to train for our dataset with $\approx 5000$ beneficiaries, Fast DEC-DFL takes minutes. This gain in speed with the added benefit of improved accuracy paves the way for DFL-based RMAB models to be deployed more widely and at larger scale. For example, this could potentially help ARMMAN with their ongoing efforts to boost engagement in their Kilkari program---the largest maternal mHealth program in the world~\cite{kilkari}, with 3 million active subscribers.

\begin{acks}
This material is based upon work supported by the NSF under Grant No. IIS-2229881. Any opinions, findings, and conclusions or recommendations expressed in this material are those of the author(s) and do not necessarily reflect the views of the NSF.
\end{acks}



\clearpage
 
\section{Ethics Statement}
\subsubsection*{Secondary Analysis and Data Usage} The experiments with the ARMMAN dataset fall into the category of secondary analysis of the aforementioned dataset. We use previously collected listenership trajectories of beneficiaries enrolled in the mMitra program. The dataset is anonymized and contains no personally identifiable information. The dataset is owned by ARMMAN and only they can share it further.

\subsubsection*{Consent for Data Collection and Sharing}
Consent for collecting data is obtained from each participant in the service call program. The data collection process is carefully explained to the participants before collecting the data. Data exchange and use were regulated through clearly defined exchange protocols including anonymization by ARMMAN, read-only access to researchers, restricted use of the data for research purposes only, and approval by ARMMAN’s ethics review committee.

\subsubsection*{Universal Accessibility of Health Information}
This study focuses on improving the effectiveness of only the \textit{live service calls}. All participants will receive the same weekly health information by automated message regardless of whether they are scheduled to receive service calls or not. The service call program does not withhold any information from the participants nor conduct any experimentation on the health information. Moreover, all participants can request service calls via a free missed call.

\subsubsection*{Road To Deployment}
The next steps involve testing our algorithm on more recent data to make sure our algorithm continues to show gains, and to run an equity audit to make sure that our algorithm prioritizes vulnerable subgroups.
We then plan to conduct a randomized field trial to evaluate the accuracy of the algorithm and verify the computational gains over the currently deployed DFL pipeline. We hope for such a model to potentially showcase its strengths in applying DFL in a cost-effective way at such a massive scale.
We must highlight, that all the above steps will be conducted with constant collaboration with ARMMAN; with ARMMAN ultimately being in charge of the actual deployment.

\bibliographystyle{ACM-Reference-Format} 
\bibliography{ref}


\begin{thebibliography}{33}


\ifx \showCODEN    \undefined \def \showCODEN     #1{\unskip}     \fi
\ifx \showDOI      \undefined \def \showDOI       #1{#1}\fi
\ifx \showISBNx    \undefined \def \showISBNx     #1{\unskip}     \fi
\ifx \showISBNxiii \undefined \def \showISBNxiii  #1{\unskip}     \fi
\ifx \showISSN     \undefined \def \showISSN      #1{\unskip}     \fi
\ifx \showLCCN     \undefined \def \showLCCN      #1{\unskip}     \fi
\ifx \shownote     \undefined \def \shownote      #1{#1}          \fi
\ifx \showarticletitle \undefined \def \showarticletitle #1{#1}   \fi
\ifx \showURL      \undefined \def \showURL       {\relax}        \fi
\providecommand\bibfield[2]{#2}
\providecommand\bibinfo[2]{#2}
\providecommand\natexlab[1]{#1}
\providecommand\showeprint[2][]{arXiv:#2}

\bibitem[\protect\citeauthoryear{Agrawal, Amos, Barratt, Boyd, Diamond, and
  Kolter}{Agrawal et~al\mbox{.}}{2019}]%
        {cvxpylayers2019}
\bibfield{author}{\bibinfo{person}{A. Agrawal}, \bibinfo{person}{B. Amos},
  \bibinfo{person}{S. Barratt}, \bibinfo{person}{S. Boyd}, \bibinfo{person}{S.
  Diamond}, {and} \bibinfo{person}{Z. Kolter}.}
  \bibinfo{year}{2019}\natexlab{}.
\newblock \showarticletitle{Differentiable Convex Optimization Layers}.
\newblock \bibinfo{journal}{\emph{Advances in Neural Information Processing
  Systems}} (\bibinfo{year}{2019}).
\newblock


\bibitem[\protect\citeauthoryear{Altman}{Altman}{1999}]%
        {altman1999constrained}
\bibfield{author}{\bibinfo{person}{Eitan Altman}.}
  \bibinfo{year}{1999}\natexlab{}.
\newblock \bibinfo{booktitle}{\emph{Constrained Markov decision processes}}.
  Vol.~\bibinfo{volume}{7}.
\newblock \bibinfo{publisher}{CRC press}.
\newblock


\bibitem[\protect\citeauthoryear{Amos and Kolter}{Amos and Kolter}{2017}]%
        {amos2017optnet}
\bibfield{author}{\bibinfo{person}{Brandon Amos} {and} \bibinfo{person}{J~Zico
  Kolter}.} \bibinfo{year}{2017}\natexlab{}.
\newblock \showarticletitle{Optnet: Differentiable optimization as a layer in
  neural networks}. In \bibinfo{booktitle}{\emph{International Conference on
  Machine Learning}}. PMLR, \bibinfo{pages}{136--145}.
\newblock


\bibitem[\protect\citeauthoryear{Amos, Koltun, and Kolter}{Amos
  et~al\mbox{.}}{2019}]%
        {amos2019limited}
\bibfield{author}{\bibinfo{person}{Brandon Amos}, \bibinfo{person}{Vladlen
  Koltun}, {and} \bibinfo{person}{J~Zico Kolter}.}
  \bibinfo{year}{2019}\natexlab{}.
\newblock \showarticletitle{The limited multi-label projection layer}.
\newblock \bibinfo{journal}{\emph{arXiv preprint arXiv:1906.08707}}
  (\bibinfo{year}{2019}).
\newblock


\bibitem[\protect\citeauthoryear{ARMMAN}{ARMMAN}{2023a}]%
        {kilkari}
\bibfield{author}{\bibinfo{person}{ARMMAN}.} \bibinfo{year}{2023}\natexlab{a}.
\newblock \bibinfo{booktitle}{\emph{Kilkari}}.
\newblock
\urldef\tempurl%
\url{https://www.armman.org/kilkari/}
\showURL{%
\tempurl}


\bibitem[\protect\citeauthoryear{ARMMAN}{ARMMAN}{2023b}]%
        {mmitra}
\bibfield{author}{\bibinfo{person}{ARMMAN}.} \bibinfo{year}{2023}\natexlab{b}.
\newblock \bibinfo{booktitle}{\emph{mMitra}}.
\newblock
\urldef\tempurl%
\url{https://www.armman.org/mmitra/}
\showURL{%
\tempurl}


\bibitem[\protect\citeauthoryear{Ayer, Zhang, Bonifonte, Spaulding, and
  Chhatwal}{Ayer et~al\mbox{.}}{2019}]%
        {ayer2019prioritizing}
\bibfield{author}{\bibinfo{person}{Turgay Ayer}, \bibinfo{person}{Can Zhang},
  \bibinfo{person}{Anthony Bonifonte}, \bibinfo{person}{Anne~C Spaulding},
  {and} \bibinfo{person}{Jagpreet Chhatwal}.} \bibinfo{year}{2019}\natexlab{}.
\newblock \showarticletitle{Prioritizing hepatitis C treatment in US prisons}.
\newblock \bibinfo{journal}{\emph{Operations Research}} \bibinfo{volume}{67},
  \bibinfo{number}{3} (\bibinfo{year}{2019}), \bibinfo{pages}{853--873}.
\newblock


\bibitem[\protect\citeauthoryear{Boyd and Vandenberghe}{Boyd and
  Vandenberghe}{2004}]%
        {boyd2004convex}
\bibfield{author}{\bibinfo{person}{Stephen~P Boyd} {and}
  \bibinfo{person}{Lieven Vandenberghe}.} \bibinfo{year}{2004}\natexlab{}.
\newblock \bibinfo{booktitle}{\emph{Convex optimization}}.
\newblock \bibinfo{publisher}{Cambridge university press}.
\newblock


\bibitem[\protect\citeauthoryear{Brent}{Brent}{2013}]%
        {brent2013algorithms}
\bibfield{author}{\bibinfo{person}{Richard~P Brent}.}
  \bibinfo{year}{2013}\natexlab{}.
\newblock \bibinfo{booktitle}{\emph{Algorithms for minimization without
  derivatives}}.
\newblock \bibinfo{publisher}{Courier Corporation}.
\newblock


\bibitem[\protect\citeauthoryear{Buchholz and Scheftelowitsch}{Buchholz and
  Scheftelowitsch}{2019}]%
        {buchholz2019computation}
\bibfield{author}{\bibinfo{person}{Peter Buchholz} {and}
  \bibinfo{person}{Dimitri Scheftelowitsch}.} \bibinfo{year}{2019}\natexlab{}.
\newblock \showarticletitle{Computation of weighted sums of rewards for
  concurrent MDPs}.
\newblock \bibinfo{journal}{\emph{Mathematical Methods of Operations Research}}
   \bibinfo{volume}{89} (\bibinfo{year}{2019}), \bibinfo{pages}{1--42}.
\newblock


\bibitem[\protect\citeauthoryear{Elmachtoub and Grigas}{Elmachtoub and
  Grigas}{2022}]%
        {elmachtoub2022smart}
\bibfield{author}{\bibinfo{person}{Adam~N Elmachtoub} {and}
  \bibinfo{person}{Paul Grigas}.} \bibinfo{year}{2022}\natexlab{}.
\newblock \showarticletitle{Smart “predict, then optimize”}.
\newblock \bibinfo{journal}{\emph{Management Science}} \bibinfo{volume}{68},
  \bibinfo{number}{1} (\bibinfo{year}{2022}), \bibinfo{pages}{9--26}.
\newblock


\bibitem[\protect\citeauthoryear{Farahmand, Barreto, and Nikovski}{Farahmand
  et~al\mbox{.}}{2017}]%
        {farahmand2017value}
\bibfield{author}{\bibinfo{person}{Amir-massoud Farahmand},
  \bibinfo{person}{Andre Barreto}, {and} \bibinfo{person}{Daniel Nikovski}.}
  \bibinfo{year}{2017}\natexlab{}.
\newblock \showarticletitle{Value-aware loss function for model-based
  reinforcement learning}. In \bibinfo{booktitle}{\emph{Artificial Intelligence
  and Statistics}}. PMLR, \bibinfo{pages}{1486--1494}.
\newblock


\bibitem[\protect\citeauthoryear{Futoma, Hughes, and Doshi-Velez}{Futoma
  et~al\mbox{.}}{2020}]%
        {futoma2020popcorn}
\bibfield{author}{\bibinfo{person}{Joseph Futoma}, \bibinfo{person}{Michael~C
  Hughes}, {and} \bibinfo{person}{Finale Doshi-Velez}.}
  \bibinfo{year}{2020}\natexlab{}.
\newblock \showarticletitle{Popcorn: Partially observed prediction constrained
  reinforcement learning}.
\newblock \bibinfo{journal}{\emph{arXiv preprint arXiv:2001.04032}}
  (\bibinfo{year}{2020}).
\newblock


\bibitem[\protect\citeauthoryear{Hawkins}{Hawkins}{2003}]%
        {hawkins2003langrangian}
\bibfield{author}{\bibinfo{person}{Jeffrey~Thomas Hawkins}.}
  \bibinfo{year}{2003}\natexlab{}.
\newblock \emph{\bibinfo{title}{A Langrangian decomposition approach to weakly
  coupled dynamic optimization problems and its applications}}.
\newblock \bibinfo{thesistype}{Ph.D. Dissertation}.
  \bibinfo{school}{Massachusetts Institute of Technology}.
\newblock


\bibitem[\protect\citeauthoryear{Hsieh, Liu, and Cevher}{Hsieh
  et~al\mbox{.}}{2019}]%
        {hsieh2019finding}
\bibfield{author}{\bibinfo{person}{Ya-Ping Hsieh}, \bibinfo{person}{Chen Liu},
  {and} \bibinfo{person}{Volkan Cevher}.} \bibinfo{year}{2019}\natexlab{}.
\newblock \showarticletitle{Finding mixed nash equilibria of generative
  adversarial networks}. In \bibinfo{booktitle}{\emph{International Conference
  on Machine Learning}}. PMLR, \bibinfo{pages}{2810--2819}.
\newblock


\bibitem[\protect\citeauthoryear{Jung and Tewari}{Jung and Tewari}{2019}]%
        {jung2019regret}
\bibfield{author}{\bibinfo{person}{Young~Hun Jung} {and} \bibinfo{person}{Ambuj
  Tewari}.} \bibinfo{year}{2019}\natexlab{}.
\newblock \showarticletitle{Regret bounds for thompson sampling in episodic
  restless bandit problems}.
\newblock \bibinfo{journal}{\emph{Advances in Neural Information Processing
  Systems}}  \bibinfo{volume}{32} (\bibinfo{year}{2019}).
\newblock


\bibitem[\protect\citeauthoryear{Killian, Jain, Jia, Amar, Huang, and
  Tambe}{Killian et~al\mbox{.}}{2023}]%
        {killian2023equitable}
\bibfield{author}{\bibinfo{person}{Jackson~A Killian}, \bibinfo{person}{Manish
  Jain}, \bibinfo{person}{Yugang Jia}, \bibinfo{person}{Jonathan Amar},
  \bibinfo{person}{Erich Huang}, {and} \bibinfo{person}{Milind Tambe}.}
  \bibinfo{year}{2023}\natexlab{}.
\newblock \showarticletitle{Equitable Restless Multi-Armed Bandits: A General
  Framework Inspired By Digital Health}.
\newblock \bibinfo{journal}{\emph{arXiv preprint arXiv:2308.09726}}
  (\bibinfo{year}{2023}).
\newblock


\bibitem[\protect\citeauthoryear{Killian, Wilder, Sharma, Choudhary, Dilkina,
  and Tambe}{Killian et~al\mbox{.}}{2019}]%
        {killian2019learning}
\bibfield{author}{\bibinfo{person}{Jackson~A Killian}, \bibinfo{person}{Bryan
  Wilder}, \bibinfo{person}{Amit Sharma}, \bibinfo{person}{Vinod Choudhary},
  \bibinfo{person}{Bistra Dilkina}, {and} \bibinfo{person}{Milind Tambe}.}
  \bibinfo{year}{2019}\natexlab{}.
\newblock \showarticletitle{Learning to prescribe interventions for
  tuberculosis patients using digital adherence data}. In
  \bibinfo{booktitle}{\emph{Proceedings of the 25th ACM SIGKDD International
  Conference on Knowledge Discovery \& Data Mining}}.
  \bibinfo{pages}{2430--2438}.
\newblock


\bibitem[\protect\citeauthoryear{Mandi, Bucarey, Tchomba, and Guns}{Mandi
  et~al\mbox{.}}{2022}]%
        {mandi2022decision}
\bibfield{author}{\bibinfo{person}{Jayanta Mandi}, \bibinfo{person}{V{\i}ctor
  Bucarey}, \bibinfo{person}{Maxime Mulamba~Ke Tchomba}, {and}
  \bibinfo{person}{Tias Guns}.} \bibinfo{year}{2022}\natexlab{}.
\newblock \showarticletitle{Decision-focused learning: through the lens of
  learning to rank}. In \bibinfo{booktitle}{\emph{International Conference on
  Machine Learning}}. PMLR, \bibinfo{pages}{14935--14947}.
\newblock


\bibitem[\protect\citeauthoryear{Mate, Killian, Xu, Perrault, and Tambe}{Mate
  et~al\mbox{.}}{2020}]%
        {mate2020collapsing}
\bibfield{author}{\bibinfo{person}{Aditya Mate}, \bibinfo{person}{Jackson
  Killian}, \bibinfo{person}{Haifeng Xu}, \bibinfo{person}{Andrew Perrault},
  {and} \bibinfo{person}{Milind Tambe}.} \bibinfo{year}{2020}\natexlab{}.
\newblock \showarticletitle{Collapsing bandits and their application to public
  health intervention}.
\newblock \bibinfo{journal}{\emph{Advances in Neural Information Processing
  Systems}}  \bibinfo{volume}{33} (\bibinfo{year}{2020}),
  \bibinfo{pages}{15639--15650}.
\newblock


\bibitem[\protect\citeauthoryear{Mate, Madaan, Taneja, Madhiwalla, Verma,
  Singh, Hegde, Varakantham, and Tambe}{Mate et~al\mbox{.}}{2022}]%
        {mate2022field}
\bibfield{author}{\bibinfo{person}{Aditya Mate}, \bibinfo{person}{Lovish
  Madaan}, \bibinfo{person}{Aparna Taneja}, \bibinfo{person}{Neha Madhiwalla},
  \bibinfo{person}{Shresth Verma}, \bibinfo{person}{Gargi Singh},
  \bibinfo{person}{Aparna Hegde}, \bibinfo{person}{Pradeep Varakantham}, {and}
  \bibinfo{person}{Milind Tambe}.} \bibinfo{year}{2022}\natexlab{}.
\newblock \showarticletitle{Field study in deploying restless multi-armed
  bandits: Assisting non-profits in improving maternal and child health}.
\newblock \bibinfo{journal}{\emph{Proceedings of the AAAI Conference on
  Artificial Intelligence}}  \bibinfo{volume}{36} (\bibinfo{year}{2022}),
  \bibinfo{pages}{12017--12025}.
\newblock


\bibitem[\protect\citeauthoryear{Nikishin, Abachi, Agarwal, and Bacon}{Nikishin
  et~al\mbox{.}}{2022}]%
        {nikishin2022control}
\bibfield{author}{\bibinfo{person}{Evgenii Nikishin}, \bibinfo{person}{Romina
  Abachi}, \bibinfo{person}{Rishabh Agarwal}, {and} \bibinfo{person}{Pierre-Luc
  Bacon}.} \bibinfo{year}{2022}\natexlab{}.
\newblock \showarticletitle{Control-oriented model-based reinforcement learning
  with implicit differentiation}. In \bibinfo{booktitle}{\emph{Proceedings of
  the AAAI Conference on Artificial Intelligence}}, Vol.~\bibinfo{volume}{36}.
  \bibinfo{pages}{7886--7894}.
\newblock


\bibitem[\protect\citeauthoryear{Papadimitriou and Tsitsiklis}{Papadimitriou
  and Tsitsiklis}{1994}]%
        {papadimitriou1994complexity}
\bibfield{author}{\bibinfo{person}{Christos~H Papadimitriou} {and}
  \bibinfo{person}{John~N Tsitsiklis}.} \bibinfo{year}{1994}\natexlab{}.
\newblock \showarticletitle{The complexity of optimal queueing network
  control}.
\newblock \bibinfo{journal}{\emph{Proceedings of IEEE 9th annual conference on
  structure in complexity Theory}} (\bibinfo{year}{1994}),
  \bibinfo{pages}{318--322}.
\newblock


\bibitem[\protect\citeauthoryear{Steimle, Kaufman, and Denton}{Steimle
  et~al\mbox{.}}{2021}]%
        {steimle2021multi}
\bibfield{author}{\bibinfo{person}{Lauren~N Steimle}, \bibinfo{person}{David~L
  Kaufman}, {and} \bibinfo{person}{Brian~T Denton}.}
  \bibinfo{year}{2021}\natexlab{}.
\newblock \showarticletitle{Multi-model Markov decision processes}.
\newblock \bibinfo{journal}{\emph{IISE Transactions}} \bibinfo{volume}{53},
  \bibinfo{number}{10} (\bibinfo{year}{2021}), \bibinfo{pages}{1124--1139}.
\newblock


\bibitem[\protect\citeauthoryear{Verma, Mate, Wang, Madhiwalla, Hegde, Taneja,
  and Tambe}{Verma et~al\mbox{.}}{2023a}]%
        {verma2023restless}
\bibfield{author}{\bibinfo{person}{Shresth Verma}, \bibinfo{person}{Aditya
  Mate}, \bibinfo{person}{Kai Wang}, \bibinfo{person}{Neha Madhiwalla},
  \bibinfo{person}{Aparna Hegde}, \bibinfo{person}{Aparna Taneja}, {and}
  \bibinfo{person}{Milind Tambe}.} \bibinfo{year}{2023}\natexlab{a}.
\newblock \showarticletitle{Restless Multi-Armed Bandits for Maternal and Child
  Health: Results from Decision-Focused Learning}.
\newblock \bibinfo{journal}{\emph{Proceedings of the 2023 International
  Conference on Autonomous Agents and Multiagent Systems}}
  (\bibinfo{year}{2023}), \bibinfo{pages}{1312--1320}.
\newblock


\bibitem[\protect\citeauthoryear{Verma, Singh, Mate, Verma, Gorantla,
  Madhiwalla, Hegde, Thakkar, Jain, Tambe, et~al\mbox{.}}{Verma
  et~al\mbox{.}}{2023b}]%
        {verma2023increasing}
\bibfield{author}{\bibinfo{person}{Shresth Verma}, \bibinfo{person}{Gargi
  Singh}, \bibinfo{person}{Aditya Mate}, \bibinfo{person}{Paritosh Verma},
  \bibinfo{person}{Sruthi Gorantla}, \bibinfo{person}{Neha Madhiwalla},
  \bibinfo{person}{Aparna Hegde}, \bibinfo{person}{Divy Thakkar},
  \bibinfo{person}{Manish Jain}, \bibinfo{person}{Milind Tambe},
  {et~al\mbox{.}}} \bibinfo{year}{2023}\natexlab{b}.
\newblock \showarticletitle{Increasing impact of mobile health programs: SAHELI
  for maternal and child care}. In \bibinfo{booktitle}{\emph{Proceedings of the
  AAAI Conference on Artificial Intelligence}}, Vol.~\bibinfo{volume}{37}.
  \bibinfo{pages}{15594--15602}.
\newblock


\bibitem[\protect\citeauthoryear{Von~Neumann and Morgenstern}{Von~Neumann and
  Morgenstern}{1947}]%
        {von1947theory}
\bibfield{author}{\bibinfo{person}{John Von~Neumann} {and}
  \bibinfo{person}{Oskar Morgenstern}.} \bibinfo{year}{1947}\natexlab{}.
\newblock \showarticletitle{Theory of games and economic behavior, 2nd rev}.
\newblock  (\bibinfo{year}{1947}).
\newblock


\bibitem[\protect\citeauthoryear{Wang, Shah, Chen, Perrault, Doshi-Velez, and
  Tambe}{Wang et~al\mbox{.}}{2021}]%
        {wang2021learning}
\bibfield{author}{\bibinfo{person}{Kai Wang}, \bibinfo{person}{Sanket Shah},
  \bibinfo{person}{Haipeng Chen}, \bibinfo{person}{Andrew Perrault},
  \bibinfo{person}{Finale Doshi-Velez}, {and} \bibinfo{person}{Milind Tambe}.}
  \bibinfo{year}{2021}\natexlab{}.
\newblock \showarticletitle{Learning mdps from features: Predict-then-optimize
  for sequential decision making by reinforcement learning}.
\newblock \bibinfo{journal}{\emph{Advances in Neural Information Processing
  Systems}}  \bibinfo{volume}{34} (\bibinfo{year}{2021}),
  \bibinfo{pages}{8795--8806}.
\newblock


\bibitem[\protect\citeauthoryear{Wang, Verma, Mate, Shah, Taneja, Madhiwalla,
  Hegde, and Tambe}{Wang et~al\mbox{.}}{2023}]%
        {wang2023scalable}
\bibfield{author}{\bibinfo{person}{Kai Wang}, \bibinfo{person}{Shresth Verma},
  \bibinfo{person}{Aditya Mate}, \bibinfo{person}{Sanket Shah},
  \bibinfo{person}{Aparna Taneja}, \bibinfo{person}{Neha Madhiwalla},
  \bibinfo{person}{Aparna Hegde}, {and} \bibinfo{person}{Milind Tambe}.}
  \bibinfo{year}{2023}\natexlab{}.
\newblock \showarticletitle{Scalable decision-focused learning in restless
  multi-armed bandits with application to maternal and child health}.
\newblock \bibinfo{journal}{\emph{Proceedings of the AAAI Conference on
  Artificial Intelligence}}  \bibinfo{volume}{37} (\bibinfo{year}{2023}),
  \bibinfo{pages}{12138--12146}.
\newblock


\bibitem[\protect\citeauthoryear{Weber and Weiss}{Weber and Weiss}{1990}]%
        {weber1990index}
\bibfield{author}{\bibinfo{person}{Richard~R Weber} {and}
  \bibinfo{person}{Gideon Weiss}.} \bibinfo{year}{1990}\natexlab{}.
\newblock \showarticletitle{On an index policy for restless bandits}.
\newblock \bibinfo{journal}{\emph{Journal of applied probability}}
  \bibinfo{volume}{27}, \bibinfo{number}{3} (\bibinfo{year}{1990}),
  \bibinfo{pages}{637--648}.
\newblock


\bibitem[\protect\citeauthoryear{Whittle}{Whittle}{1988}]%
        {whittle1988restless}
\bibfield{author}{\bibinfo{person}{Peter Whittle}.}
  \bibinfo{year}{1988}\natexlab{}.
\newblock \showarticletitle{Restless bandits: Activity allocation in a changing
  world}.
\newblock \bibinfo{journal}{\emph{Journal of applied probability}}
  \bibinfo{volume}{25}, \bibinfo{number}{A} (\bibinfo{year}{1988}),
  \bibinfo{pages}{287--298}.
\newblock


\bibitem[\protect\citeauthoryear{Wilder, Dilkina, and Tambe}{Wilder
  et~al\mbox{.}}{2019}]%
        {wilder2019melding}
\bibfield{author}{\bibinfo{person}{Bryan Wilder}, \bibinfo{person}{Bistra
  Dilkina}, {and} \bibinfo{person}{Milind Tambe}.}
  \bibinfo{year}{2019}\natexlab{}.
\newblock \showarticletitle{Melding the data-decisions pipeline:
  Decision-focused learning for combinatorial optimization}. In
  \bibinfo{booktitle}{\emph{Proceedings of the AAAI Conference on Artificial
  Intelligence}}, Vol.~\bibinfo{volume}{33}. \bibinfo{pages}{1658--1665}.
\newblock


\bibitem[\protect\citeauthoryear{Yanovskaya}{Yanovskaya}{1974}]%
        {yanovskaya1974infinite}
\bibfield{author}{\bibinfo{person}{EB Yanovskaya}.}
  \bibinfo{year}{1974}\natexlab{}.
\newblock \showarticletitle{Infinite zero-sum two-person games}.
\newblock \bibinfo{journal}{\emph{Journal of Soviet Mathematics}}
  \bibinfo{volume}{2}, \bibinfo{number}{5} (\bibinfo{year}{1974}),
  \bibinfo{pages}{520--541}.
\newblock


\end{thebibliography}


\clearpage
\onecolumn
\appendix

\section{Efficiently Calculating the Returns of Decomposed Policy}\label{sec:returns}

\begin{center}
\begin{minipage}{0.4\textwidth}
\begin{algorithm}[H]
\caption{\textsc{GetReturns}}
\label{alg:returns}
\raggedright
\textbf{Input}: Transition matrices $T$, Rewards $R$, Policy $\pi$\\
\textbf{Output}: Expected return $J_{T}(\pi)$
\begin{algorithmic}[1] 
\State Get the markov transitions induced by the policy $\pi$:
$$T_\pi(s, s') \gets T(s, \pi(s), s')$$
\State Get the corresponding value function:
$$V \gets (I - \gamma T_\pi)^{-1} R$$
\State Multiply $V$ with the initial state distribution:
$$J_{T}(\pi) \gets \E_{s_0} [V(s_0)]$$
\State \textbf{return} $J_{T}(\pi)$
\end{algorithmic}
\end{algorithm}
\end{minipage}
\end{center}

\section{Proof of Theorem~\ref{thm:badminima}}
For the sake of clarity, we restate the Theorem~\ref{thm:badminima} below.
\begin{thm*}
  Predicting $\bmthat = \bmtruet$ is \textbf{not} always a maximizer of the Predict-Then-Optimize problem below:
    \begin{align*}
        \textcolor{myOrange}{\bmthat^\star} = \argmax_{\bmthat}\;  J_{\bmtruet}(\bm{\pi}^{\that\text{-DEC}}(\bmthat))
    \end{align*}
\end{thm*}
\begin{proof}
    Consider a 2-state RMAB with $2$ arms, $\gamma = 0.9$, and a budget $B = \frac{1}{1 + \gamma}$ (i.e., expected budget $\frac{B}{1 - \gamma} = \frac{1}{1 - \gamma^2}$). One arm has a transition matrix described by $\tgood$ and the other by $\tbad$:
    \begin{align*}
        \tgood = \left[
        \begin{array}{@{}cc|cc@{}}
            1 & 0 & 0 & 1 \\
            1 & 0 & 1 & 0
        \end{array}\right],\;
        \tbad = \left[
        \begin{array}{@{}cc|cc@{}}
            1 & 0 & 0.5 & 0.5 \\
            1 & 0 & 1 & 0
        \end{array}\right]
    \end{align*}

    Now, acting in state 1 for either $\tgood$ or $\tbad$ (lower row) doesn't make sense because there's no difference in the transition probabilities whether you act or not. Acting in state 0 of $\tgood$ uses an expected budget of $\jbar = \frac{1}{1 - \gamma^2}$ and increases the expected return $\Delta J_{\tgood}$ by $\frac{\gamma}{1 - \gamma^2}$. Acting in state 0 of $\tbad$ uses an expected budget of $\jbar = \frac{2}{2 - \gamma - \gamma^2}$ and increases the expected return $\Delta J_{\tgood}$ by $\frac{\gamma}{2 - \gamma - \gamma^2}$. So, if we solve for $\bm{\pi}^{{\that\text{-DEC}}}([\tgood,\tbad])$, the policy we get will be to only act in state 0 of $\tgood$, because (a) it has a higher ratio of $\frac{\Delta J}{\jbar}$ than acting in state 0 of $\tbad$, and (b) uses up all the budget.

    However, if we'd instead predicted the ``best-case'' transition matrix $T^{\text{OPT}}$ as defined in \cref{eqn:bestpred}, we could do better. As discussed in \cref{ex:violation}, acting in state 0 of $T^{\text{OPT}}$ only uses an expected budget of $\jbar = 1$. Therefore, solving for $\bm{\pi}^{{\that\text{-DEC}}}([T^{\text{OPT}},T^{\text{OPT}}])$ results in a policy for acting in state 0 for both $\tgood$ and $\tbad$ (as long as $2 < \frac{1}{1 - \gamma^2}$, which is satisfied for $\gamma = 0.9$). This is strictly better than $\bm{\pi}^{{\that\text{-DEC}}}([\tgood,\tbad])$ which only acts in state 0 of $\tgood$. Therefore:
    \begin{align*}
        J_{\bmtruet}(\bm{\pi}^{{\that\text{-DEC}}}([T^{\text{OPT}},T^{\text{OPT}}]) > J_{\bmtruet}(\bm{\pi}^{{\that\text{-DEC}}}([\tgood,\tbad]) & \qedhere
    \end{align*}
\end{proof}
Note that there isn't anything special about our choice of $\bmtruet$; we just chose values that simplify the exposition. We could, however, repeat this sort of argument for almost any choice of $\truet$ where acting is better than not acting! 

\section{Proof of Theorem~\ref{thm:decomposable}}
For the sake of clarity, we restate the Theorem~\ref{thm:decomposable} below.
\begin{thm*}
    Let $\Omega$ be the set of all distributions over deterministic policies, and let $\Omega^{\mathrm{DEC}}$ be the set of all distributions over deterministic, decomposable policies. Consider the following optimization problems
    \begin{align}
    \label{eqn:decomposable_proof1}
        \max_{Z \in \Omega} \quad  \E_{\bm{\pi} \sim Z} [J_{\bmtruet}(\bm{\pi})], \quad \text{s.t.}  \;\;  \E_{\bm{\pi} \sim Z }[\jbar_{\bmtruet}(\bm{\pi})] \leq \frac{B}{1-\gamma} \\
    \label{eqn:decomposable_proof2}
         \max_{\mathclap{Z \in \Omega^{\mathrm{DEC}}}} \quad \E_{\bm{\pi} \sim Z} [J_{\bmtruet}(\bm{\pi})],\quad \text{s.t.}  \;\; \E_{\bm{\pi} \sim Z }[\jbar_{\bmtruet}(\bm{\pi})] \leq \frac{B}{1-\gamma}
    \end{align}
    Then, any maximizer of optimization problem~\eqref{eqn:decomposable_proof2} is also a maximizer of optimization problem~\eqref{eqn:decomposable_proof1}.
\end{thm*}
\begin{proof}
The Lagrangian of the first optimization problem is given by
\begin{align*}
    \max_{Z \in \Omega} \min_{\lambda \geq 0} \; \E_{\bm{\pi} \sim Z} [J_{\bmtruet}(\bm{\pi})]  + \lambda\left( \frac{B}{1-\gamma} - \E_{\bm{\pi} \sim Z }[\jbar_{\bmtruet}(\bm{\pi})]\right)
\end{align*}
Note that the above objective is linear in $\lambda$ and $Z$. Using popular minimax theorems we can swap the ordering of min and max and obtain the following equivalent problem~\citep{von1947theory, yanovskaya1974infinite}
\begin{align*}
     \min_{\lambda \geq 0} \max_{Z \in \Omega}\; \E_{\bm{\pi} \sim Z} [J_{\bmtruet}(\bm{\pi})]  + \lambda\left( \frac{B}{1-\gamma} - \E_{\bm{\pi} \sim Z }[\jbar_{\bmtruet}(\bm{\pi})]\right).
\end{align*}
Observe that for any fixed $\lambda$, the inner optimization decomposes across the $N$ arms. Using this observation, it is easy to see that there exists an optimal $Z$ that decomposes across the arms. So, the above problem can be equivalently written as
\begin{align*}
     \min_{\lambda \geq 0} \max_{Z \in \Omega^{\mathrm{DEC}}}\; \E_{\bm{\pi} \sim Z} [J_{\bmtruet}(\bm{\pi})]  + \lambda\left( \frac{B}{1-\gamma} - \E_{\bm{\pi} \sim Z }[\jbar_{\bmtruet}(\bm{\pi})]\right).
\end{align*}
By appealing to  minimax theorems we again swap the ordering of min and max and obtain the following equivalent problem
\begin{align*}
     \max_{Z \in \Omega^{\mathrm{DEC}}}\min_{\lambda \geq 0} \; \E_{\bm{\pi} \sim Z} [J_{\bmtruet}(\bm{\pi})]  + \lambda\left( \frac{B}{1-\gamma} - \E_{\bm{\pi} \sim Z }[\jbar_{\bmtruet}(\bm{\pi})]\right).
\end{align*}
Note that this is equivalent to the second problem in Equation~\eqref{eqn:decomposable_proof2}. This shows that any optimizer of Equation~\eqref{eqn:decomposable_proof2} is also an optimizer of Equation~\eqref{eqn:decomposable_proof1}.
\end{proof}

\section{Additional Results}
\begin{proposition}
\label{prop:softmax_monotonicity}
Let $f(\lambda) = \frac{\sum_{i=}^Nv_i e^{-\lambda v_i}}{\sum_{i=1}^N e^{-\lambda v_i}}$. Then $f$ is a monotonically decreasing function of $\lambda$.
\end{proposition}
\begin{proof}
    The first derivative of $f$ is given by
    \[
    f'(\lambda) = -\frac{\sum_{i=1}^N v_i^2e^{-\lambda v_i}}{\sum_{i=1}^N e^{-\lambda v_i}} + \frac{(\sum_{i=1}^N v_ie^{-\lambda v_i})^2}{(\sum_{i=1}^N e^{-\lambda v_i})^2}.
    \]
    From the definition of $f(\lambda),$ the derivative can be rewritten as
    \[
    f'(\lambda) = -\frac{\sum_{i=1}^N (v_i-f(\lambda))^2e^{-\lambda v_i}}{\sum_{i=1}^N e^{-\lambda v_i}}
    \]
    This shows that $f'(\lambda) \leq 0$. Consequently, $f'$ is a decreasing function of $\lambda$. If at least one $v_i$ is different from others, then $f'(\lambda) < 0$ and $f$ is a strictly decreasing function of $\lambda$.
\end{proof}

\section{Visualizing the Learned Models}\label{sec:viz}
We visualize the predictions of the learned models in \cref{fig:viz}. We plot the predicted Whittle Index versus the true Whittle Index. In \cref{fig:mse_viz}, there isn't much difference between the Whittle Index distribution in the blue shaded region versus the population, highlighting that the model is not able to isolate beneficiaries for whom the action effect would be high. Conversely, in \cref{fig:dec_dfl_viz}, we see that the Whittle indices in the blue region have high true values, implying good model performance. We find that the model in \cref{fig:dfl_viz} has performance somewhere in between (a) and (c). This shows that our approach is able to effectively find subsets of the population.

\begin{figure*}[b]
\centering
\begin{subfigure}[b]{0.45\textwidth}
    \centering
    \includegraphics[width=\linewidth]{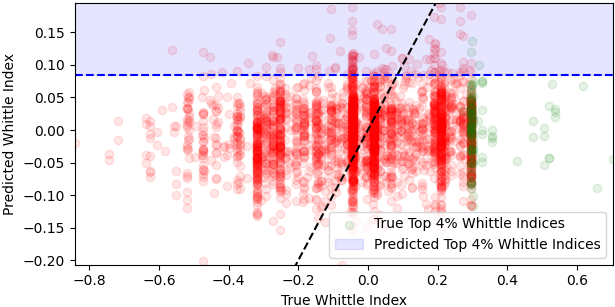}
    \vspace{-1.2em}
    \caption{MSE (2-Stage)}
    \label{fig:mse_viz}
\end{subfigure}
\hfil
\begin{subfigure}[b]{0.45\textwidth}
    \centering
    \includegraphics[width=\linewidth]{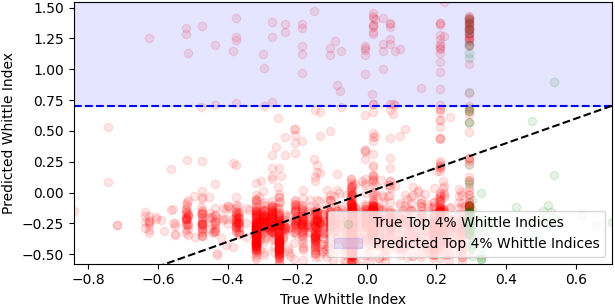}
    \vspace{-1.2em}
    \caption{DFL \cite{wang2023scalable}, 100 Trajectories}
    \label{fig:dfl_viz}
\end{subfigure}
\\\vspace{0.5em}
\begin{subfigure}[b]{0.5\textwidth}
    \centering
    \includegraphics[width=\linewidth]{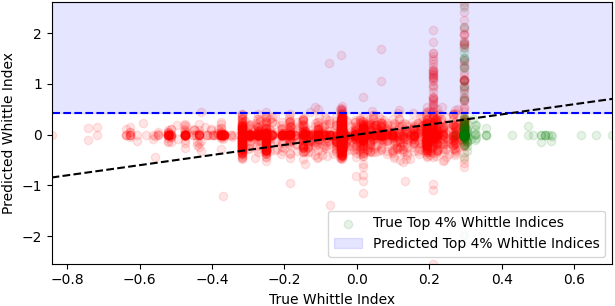}  
    \vspace{-1.2em}
    \caption{DEC-DFL (Ours)}
    \label{fig:dec_dfl_viz}
\end{subfigure}
\caption{Visualization of Predictions on Real-World Domain. \textmd{We plot the \textit{predicted} (y-axis) versus \textit{true} (x-axis) Whittle indices induced by different loss functions. A good loss function is one for which the top-$B$ predicted Whittle indices (blue shaded region) are actually high, i.e., the \textit{true} action effect is high when we predict a high action effect.}}
\label{fig:viz}
\end{figure*}

\end{document}